\documentclass{article}

     \usepackage[preprint]{neurips_2023}

\usepackage[utf8]{inputenc} 
\usepackage[T1]{fontenc}    
\usepackage{hyperref}       
\usepackage{url}            
\usepackage{booktabs}       
\usepackage{amsfonts}       
\usepackage{nicefrac}       
\usepackage{microtype}      
\usepackage{xcolor}         

\usepackage{amsthm}
\usepackage{lipsum}
\usepackage{graphicx}
\graphicspath{ {./images/} }

\usepackage{comment}
\usepackage{mathrsfs}
\usepackage{amsmath,amssymb,amsfonts,amsbsy}
\usepackage{algorithm}
\usepackage{algpseudocode}
\usepackage{comment}
\newtheorem{theorem}{Theorem}

\newtheorem{lemma}{Lemma}

\newtheorem{fact}{Fact}
\newtheorem{remark}{Remark}
\newtheorem{assumption}{Assumption}

\newcommand{\abs}[1]{\left\lvert #1 \right\rvert}
\newcommand{\norm}[1]{\left\lVert #1 \right\rVert}
\newcommand{\brk}[1]{\left[ #1 \right]}
\newcommand{\cbrk}[1]{\left\{ #1 \right\}}
\newcommand{\prt}[1]{\left( #1 \right)}

\newcommand{\tr}[1]{\textcolor{red}{#1}}

\newcommand{\cF}{\mathcal{F}}

\newcommand{\cM}{\mathcal{M}}
\newcommand{\cS}{\mathcal{S}}
\newcommand{\cA}{\mathcal{A}}

\newcommand{\cO}{\mathcal{O}}
\newcommand{\sA}{\mathscr{A}}
\newcommand{\sD}{\mathscr{D}}
\newcommand{\bP}{\mathbb{P}}
\newcommand{\bE}{\mathbb{E}}
\newcommand{\bR}{\mathbb{R}}
\newcommand{\bI}{\mathbb{I}}

\newcommand{\re}[1]{\text{Regret}}

\usepackage{caption}
\usepackage{subcaption}

\title{Regret Bounds for Risk-sensitive Reinforcement Learning with Lipschitz Dynamic Risk Measures}

\author{%
  Hao Liang
    \\
  School of Science and Engineering\\
  The Chinese University of Hong Kong, Shenzhen\\
  \texttt{217019008@link.cuhk.edu.cn} \\
\And
  Zhi-quan Luo
    \\
  School of Science and Engineering\\
  The Chinese University of Hong Kong, Shenzhen\\
  \texttt{luozq@cuhk.edu.cn} \\
}

\begin{document}

\maketitle

\begin{abstract}
We study finite episodic Markov decision processes incorporating dynamic risk measures  to capture risk sensitivity. To this end, we present two model-based algorithms applied to  \emph{Lipschitz} dynamic risk measures, a wide range of risk measures that subsumes spectral risk measure, optimized certainty equivalent, distortion risk measures among others. We establish both regret upper bounds and lower bounds. Notably, our upper bounds demonstrate optimal dependencies on the number of actions and episodes, while  reflecting the inherent trade-off between risk sensitivity and sample complexity. Additionally, we substantiate our theoretical results through numerical experiments.
\end{abstract}

\section{Introduction}
Standard reinforcement learning (RL) aims to identify an optimal policy that maximizes the expected return \citep{sutton2018reinforcement}. This approach is commonly known as risk-neutral RL since it prioritizes the mean value of the uncertain return. However, in domains characterized by high-stakes scenarios, such as finance \citep{davis2008risk,bielecki2000risk}, medical treatment \citep{ernst2006clinical}, and operations research \citep{delage2010percentile}, decision-makers exhibit risk-sensitive behavior and strive to optimize a risk measure associated with the return.

One classical framework that addresses risk sensitivity in Markov decision processes (MDPs) is the \emph{static risk measure}. In this framework, the value of a policy is defined as the static risk measure applied to the cumulative reward across all stages.  Among the commonly used static risk measures are the entropic risk measure (ERM)\cite{howard1972risk,follmer2011entropic} and the conditional value at risk (CVaR) \cite{rockafellar2000optimization}, along with several others. However,  except for the ERM, the static risk measure generally does not satisfy the Bellman equation. Consequently, obtaining the optimal policy becomes computationally challenging, even when the MDP model is known.

As an extension of the static risk measure, the \emph{dynamic risk measure} (DRM) \cite{ruszczynski2010risk} is constructed by recursively applying the risk measure to the reward at each stage. This recursive formulation naturally allows for the derivation of a dynamic programming equation and thus circumvents the computational burden. Furthermore, DRMs have the advantage of yielding time-consistent optimal policies, a property that is particularly justified in financial applications \cite{osogami2012iterated}. By ensuring time consistency, DRMs provide a more robust framework for decision-making in safety-critical applications, such as clinical treatment, where risk sensitivity at all stages is of paramount importance \cite{duprovably}.

Our work focuses on studying risk-sensitive reinforcement learning (RSRL) with a general DRM in the tabular and episodic MDP setting, in which the agent interacts with an unknown MDP with finite states and actions in an episodic manner. We make the mildest assumption that the risk measure used is Lipschitz continuous with respect to certain metric, which we refer to as the Lipschitz DRM. The Lipschitz risk measure encompasses a wide range of classes of  risk measures in practical applications, including spectral risk measure (SRM), distortion risk measure, and optimized certainty equivalent (OCE), among others. Additionally, the Lipschitz DRM is also a broader class of risk measures compared to convex and coherent measures since any finite convex risk measure satisfies the Lipschitz property \cite{follmer2013convex}. As a result, our framework encompasses various RL settings, such as risk-neutral RL, RSRL with ERM, RSRL with dynamic CVaR, and RL with dynamic OCE \footnote{Further details on these RL settings can be found in Section \ref{sec:lit}}.

The use of Lipschitz risk measures introduces additional technical challenges that need to be addressed. One such challenge arises in the algorithmic design phase when designing exploration bonuses for generic nonlinear risk measures. The standard techniques commonly used in risk-neutral settings, such as Hoeffding inequality or Bernstein-type concentration bounds, are not directly applicable as they only deal with the concentration of mean values. To overcome this issue, previous works, such as  \cite{duprovably} and  \cite{xu2023regret}, design the exploration bonus based on specific properties of the risk measures they consider. For instance, \cite{duprovably} chooses the exploration bonus for dynamic CVaR based on a classical concentration bound specific to CVaR, while \cite{xu2023regret} exploits the optimization representation of OCE and uses the concavity of the utility function to construct the bonus for OCE. We exploit the Lipschitz property of the risk measure to relate the value difference to the supremum distance and then apply DKW inequality to bound the deviation between one distribution and its empirical version.

Another challenge arises in deriving a recursion of the suboptimality gap across stages in the proof of regret upper bounds. The standard analysis in the risk-neutral setting relies on the linearity of the mean to obtain the recursion, which cannot be directly adapted to the risk-sensitive setting. To tackle this challenge, we leverage the Lipschitz property  to establish a relationship between the value difference and the Wasserstein distance between two probability distributions. Specifically, we bound the Wasserstein distance between two probability mass functions (PMFs) with identical probability mass but different support by the expected difference between their supports. By incorporating the Lipschitz property, we obtain a recursion of the suboptimality gap, where the Lipschitz constant appears as a multiplicative factor. By addressing these challenges, we are able to design efficient algorithms and provide regret upper bounds.

We summarize our main contributions in this paper as follows:

\textbf{1.} We propose two model-based algorithms for RSRL with Lipschitz DRM. These algorithms incorporate the principle of optimism in the face of uncertainty (OFU) in different ways to facilitate efficient learning. To the best of our knowledge, this is the first work that investigates  RSRL using general DRM without making the simulator assumption.

\textbf{2.} We provide  worst-case and gap-dependent  regret upper bounds for the proposed algorithms. 
Notably, the regret bounds are optimal in terms of the number of actions ($A$), and the number of episodes ($K$). They are dependent on the product of the Lipschitz constants of the risk measures at all stages, capturing the inherent trade-off between risk sensitivity and sample complexity. 

\textbf{3.} We establish the minimax and gap-dependent lower bounds for episodic MDPs with general DRM. These lower bounds are tight  in terms of $A$, $K$, and the number of states ($S$). Moreover, they reveal a constant factor that depends on the specific risk measure employed.
\subsection{Related Work}
\label{sec:lit}
\paragraph{RSRL without regret bounds.} 
General DRM applied to MDP  is presented in  \cite{ruszczynski2010risk,shen2013risk,chu2014markov,asienkiewicz2017note,bauerle2022markov}.  However, these works typically assume that the model of the MDP is known, whereas our paper focuses on studying regret guarantees for RSRL in the presence of an unknown MDP. While there are studies such as \cite{coache2021reinforcement,coache2022conditionally} that explore RSRL with dynamic convex risk measures and dynamic spectral risk measures, respectively, their work does not consider regret guarantees.

\paragraph{Regret bounds for RSRL with static risk measures.} \cite{fei2020risk} provide the first regret bound for risk-sensitive tabular MDPs using the ERM. This result is further improved upon by  \cite{fei2021exponential}, where they remove the exponential factor dependence on the episode length.  \cite{fei2022cascaded}  present the first gap-dependent regret bounds  under this framework.  \cite{liang2022bridging} propose distributional reinforcement learning  algorithms for RSRL with ERM, matching the results obtained in \cite{fei2021exponential}. \cite{bastani2022regret} consider RSRL with the objective of the spectral risk measure, where conditional value at risk (CVaR) is a special case. 
Furthermore, \cite{wang2023near} improve upon the regret bound obtained in \cite{bastani2022regret} in terms of the number of states and episode length.

\paragraph{Regret bounds for RSRL with DRMs.}  \cite{duprovably} provides the first regret bound for RSRL using DRMs, specifically focusing on dynamic conditional value at risk (CVaR). A very recent  work  \cite{xu2023regret} that investigates RSRL using dynamic OCE. OCE is a class of risk measures that encompasses several well-known measures, including ERM, CVaR, and mean-variance. They propose a value iteration algorithm based on the idea of upper confidence bound and derive regret upper bounds for their algorithm, as well as a minimax lower bound.

\cite{lamrisk} focuses on dynamic coherent risk measures in the context of non-linear function approximation. They propose an algorithm that leverages UCB-based value functions with non-linear function approximation and prove a sub-linear regret upper bound. However, their work relies on the assumption of a weak simulator, which allows for generating an arbitrary number of next states from any given state. It remains unclear whether such assumptions can be removed in the tabular setting. In contrast, our work does not rely on specific assumptions about the risk measure estimator or concentration bounds. Additionally, our approach considers a broader class of DRMs by focusing on Lipschitz DRMs, which encompasses a wider range of risk measures compared to coherent ones.

Our paper is organized as follows. We first introduce some backgrounds and problem formulations in Section \ref{sec:pre}. We then propose our algorithms in Section \ref{sec:alg}, which is followed by our main results in Section \ref{sec:main}. The proof sketch of our main theorem in given in Section \ref{sec:pf}. Finally, we conclude our paper in Section \ref{sec:con}. The numerical experiments are shown in Appendix \ref{app:exp}.

\section{Preliminaries}
\label{sec:pre}
\paragraph{Notations.} We write $[N]:=\{1,2,...,N\}$ for any positive integers $N$.  
We use $\mathbb{I}\{\cdot\}$  to denote the indicator function.  We denote by $a\vee b:=\max\{a,b\}$.  We use the notation $\tilde{\mathcal{O}}(\cdot)$ to represent $\mathcal{O}(\cdot)$ with logarithmic factors omitted. For two real numbers $a<b$, the notation $\mathscr{D}([a,b])$ refers to the space of all probability distributions that are bounded over the interval $[a,b]$. For a discrete set $x=\cbrk{x_1,\cdots,x_n}$ and a probability vector $P=(P_1,\cdots,P_n)$, the notation $(x,P)$ represents the discrete distribution where $\bP(X=x_i)=P_i$.
\paragraph{Static risk measure.}
A (static) risk measure is a mapping $\rho:\mathscr{X}\rightarrow\mathbb{R}$ that assigns a real number to each random variable in the set $\mathscr{X}$, which satisfies certain properties of the following. It quantifies the risk  associated with a random outcome. 
\begin{itemize}
    \item  monotonicity:  $X \preceq Y \Rightarrow \rho(X) \leq \rho(Y)$,
    \item  translation-invariance: $\rho(X+c)=\rho(X)+c,\  c\in\mathbb{R}$,
    \item super-additivity: $\rho(X+Y)\ge \rho(X)+\rho(Y)$,
    \item positive homogeneity: $\rho(\alpha X)=\alpha \rho(X)$ for $\alpha\ge0$
    \item concavity: $\rho(\alpha X+(1-\alpha) Y)\ge \alpha \rho(X)+(1-\alpha)\rho(Y)$
    \item  law-invariance: $F_{X}=F_{Y} \Rightarrow \rho(X)=\rho(Y)$.
\end{itemize}
Two intrinsic properties of risk measures are monotonicity and translation-invariance. 
\emph{Coherent} risk measures, introduced by \cite{artzner1999coherent}, are a widely used class of risk measures that satisfy  super-additivity and positive homogeneity in addition. 
Coherent risk measures capture important concepts such as diversification and risk pooling. \emph{Concave} risk measures generalize coherent risk measures by relaxing the requirements of super-additivity and positive homogeneity to concavity. 
Concave risk measures are more flexible and can capture a wider range of risk preferences.

\emph{Lipschitz} risk measures, on the other hand, form an even broader class of risk measures, which encompass both coherent and concave risk measures. They allow for more general functional forms and provide a flexible framework for capturing risk in various settings. Lipschitz risk measures satisfy the law-invariance property, therefore we overload notations and write $\rho(F_X):=\rho(X)$ for $X\sim F_X$.
\paragraph{Lipschitz continuity.} 
For two cumulative distribution functions (CDFs) $F$ and $G$, their supremum distance is defined as
\[ \norm{F-G}_{\infty}\triangleq \sup_{x\in\bR}\abs{F(x)-G(x)}. \]
The Wasserstein distance as well as the $\ell_1$ distance between $F$ and $G$ is defined as
\begin{align*}
    \norm{F-G}_1\triangleq\int_{-\infty}^{\infty}\abs{F(x)-G(x)}dx.
\end{align*}
A risk measure $\rho$ is said to be Lipschitz continuous with respect to a distance  $\norm{\cdot}_p$ ($p=1$ or $p=\infty$) over the set of distributions $\sD([a,b])$  if 
\[ \rho(F)-\rho(G) \leq L_{p}(\rho,[a,b]) \norm{F-G}_p, \forall F,G \in \sD([a,b]).\]
Here, $L_p(\rho,[a,b])$ is the Lipschitz constant associated with the risk measure $\rho$ over the interval $[a,b]$, and it represents the maximum rate of change of the risk measure with respect to the distance metric. The Lipschitz continuity property provides a way to quantify the sensitivity of the risk measure to changes in the underlying distributions. A larger Lipschitz constant indicates a greater sensitivity or variability of the risk measure values with respect to changes in the distributions.

To gain some intuition, we present the Lipschitz constant values for several popular risk measures over the interval $[0,M]$ in Table \ref{tab:lip}. For interested readers, please refer to  Appendix \ref{app:rm} for the formal definitions of the risk measures and more detailed discussions about the Lipschitz constants.
\begin{table}
	\caption{Lipschitz constants of typical risk measures}	
	\label{tab:lip}
	\centering
	\begin{tabular}{ccccc }
		\hline
		Lipschitz constant & CVaR  &distortion risk measure &ERM &OCE\footnote{For convenience, we only consider the OCE with concave utility function} \\ [0.5ex] 
		\hline
		$L_1([0,M])$ & $\frac{1}{\alpha}$      &$\max g'(x)$ & $\exp(|\beta| M)$& $ u'(-M)$ \\ [0.5ex] 
		$L_{\infty}([0,M])$ & $\frac{M}{\alpha}$     &$M\max g'(x)$ & $\frac{\exp(|\beta| M)-1}{|\beta| }$& $u(-M)$\\
		\hline
	\end{tabular}
\end{table} 
\paragraph{Episodic MDP.}
An episodic MDP is defined by a tuple  $\mathcal{M}\triangleq(\mathcal{S},\mathcal{A},(P_h)_{h\in[H]},(r_h)_{h\in[H]},H)$, where $\mathcal{S}$ is the finite state space with cardinality $S\triangleq|\cS|$, $\mathcal{A}$ the finite action space with cardinality $A\triangleq|\cA|$, $P_h:\mathcal{S}\times\mathcal{A}\times\mathcal{S}\rightarrow[0,1]$ the probability transition kernel at step $h$, $r_h:\mathcal{S}\times\mathcal{A}\rightarrow[0,1]$  the  reward functions at step $h$, and $H$  the length of one episode. The agent interacts with the environment for $K$ episodes. At the beginning of episode $k$, an initial state $s^k_1$ is arbitrarily selected. In step $h$, the agent takes action $a^k_h$ based on the state $s^k_h$, according to its policy. The policy is represented by a (deterministic) sequence of functions $\pi=(\pi_h)_{h\in[H]}$, where each $\pi_h$ maps from $\mathcal{S}$ to $\mathcal{A}$. The agent observes the reward $r_h(s^k_h,a^k_h)$ and transitions to the next state $s^k_{h+1}\sim P_h(\cdot|s^k_h,a^k_h)$. The episode terminates at $H+1$, after which the agent proceeds to the next episode. 
\paragraph{Dynamic programming with DRM.}
The dynamic risk measure is defined via a recursive application of static risk measures $(\rho_h)_{h\in[H-1]}$ \cite{ruszczynski2010risk}. 
The (risk-sensitive) value function of a policy $\pi$ at step $h$  is defined recursively
\begin{equation}
    \label{eqt:policy_eval}
    \begin{aligned}
    &Q^{\pi}_h(s_h,a_h)= r_h(s_h,a_h)+\rho_h(V^{\pi}_{h+1}(s_{h+1})) \\
    &V^{\pi}_h(s_h)=Q^{\pi}_h(s_h,\pi_h(s_h)), V^{\pi}_{H+1}(s_{H+1})=0.
\end{aligned}
\end{equation}
where $\rho_h$ is taken over the \emph{next-state value} $V^{\pi}_{h+1}(s_{h+1})$, i.e.,
\[ V^{\pi}_{h+1}(s_{h+1}) \sim \prt{V^{\pi}_{h+1}, P_h(s, a)} \Longrightarrow \rho_h(V^{\pi}_{h+1}(s_{h+1}))=\rho_h\prt{\prt{V^{\pi}_{h+1}, P_h(s, a)}}. \]
We refer to the distribution of $V^{\pi}_{h+1}(s_{h+1})$ as the \emph{(next-state) value distribution} $\prt{V^{\pi}_{h+1}, P_h(s, a)}$. For convenience, we write $\rho(x,P)=\rho((x,P))$, thus we write $\rho_h\prt{V^{\pi}_{h+1}, P_h(s, a)}$. By incorporating the risk measure $\rho_h$ into the recursive formulation, the dynamic risk measure framework provides a way to account for risk preferences and evaluate the risk-sensitive value function of a policy at each time step. When the risk measure $\rho_h$ specializes in the mean (i.e., taking the expectation), Equation \ref{eqt:policy_eval} reduces to the standard Bellman equation.

The (risk-sensitive) optimal policy is defined as the policy that maximizes the value function, i.e., $\pi^*=\arg\max_{\pi} V^{\pi}_1$. Consequently, the optimal value function is defined as $V^*_h(s)=V^{\pi^*}_h(s)$ and $Q^*_h(s,a)=Q^{\pi^*}_h(s,a)$.
\cite{ruszczynski2010risk} shows that an optimal Markovian policy exists, and the optimal value functions can be computed recursively. The Bellman optimality equation is given by
\begin{equation}
    \label{eqt:opt}
    \begin{aligned}
    &Q^{*}_h(s_h,a_h)= r_h(s_h,a_h)+\rho_h(V^{*}_{h+1}, P_h(s_h,a_h)) \\
    &V^{*}_h(s_h)=\max_{a\in \cA} Q^{*}_h(s_h,a_h), V^{*}_{H+1}(s_{H+1})=0.
\end{aligned}
\end{equation}
The optimal policy is the greedy policy with respect to the optimal action-value function, i.e., $\pi^*_h(s)=\arg\max_{a\in \cA} Q^{*}_h(s, a)$.
\paragraph{Regret.}
We define the regret of an algorithm \texttt{alg}   interacting with an MDP  $\mathcal{M}$ for $K$ episodes as
\begin{align*}
    \text{Regret}(\texttt{alg},\mathcal{M},K)&\triangleq \sum_{k=1}^K V^*_1(s^k_1)-V^{\pi^k}_1(s^k_1).
\end{align*}
The regret quantifies the accumulated suboptimality gap of an algorithm  compared to  the optimal policy. It is a random variable due to the randomness  introduced by $\pi^k$. We denote the expected regret by $\mathbb{E}[\text{Regret}(\texttt{alg},\mathcal{M},K)]$. We may omit the notation  $\boldsymbol{\pi}$ and $\mathcal{M}$ when clear from the context.

\section{Algorithm}
\label{sec:alg}
In this section, we present two model-based algorithms that incorporate the OFU principle. They aim to strike a balance between exploration and exploitation during the learning process. Both algorithms belong to the model-based algorithm class as they maintain an empirical model of the environment during the learning process.  We make the following assumption on the DRM, which our algorithms apply to:
\begin{assumption}
\label{asp:rm}
For each $h\in[H]$, $\rho_h$ is Lipschitz continuous with respect to the $\norm{\cdot}_1$ and $\norm{\cdot}_{\infty}$, and  satisfies $\rho_h:\sD([a,b])\rightarrow[a,b]$.
\end{assumption}
\begin{remark}
    The second condition in Assumption \ref{asp:rm} is mild since it is satisfied by common risk measures. For example, it is easy to check that CVaR and ERM satisfies this condition.
 \end{remark}
For simplicity, we drop $\rho$ from the notations and write $L_{p,h}\triangleq L_{p}(\rho_h,[0,H-h])$ for $h\in[H-1]$. For two probability mass functions (PMFs) $P$ and $Q$ with the same support, we overload notations and denote by $\norm{P-Q}_1:=\sum_{i}|P_1-Q_i|$ their $\ell_1$ distance.
\subsection{\texttt{UCBVI-DRM}}
\texttt{UCBVI-DRM} uses the bonus term to ensure optimism in the estimation of the value function, considering the nonlinearity of the risk measure. In each step $h$ of episode $k$, the optimistic value function is obtained by adding a bonus term $b^k_h$ to the empirical value. The empirical value is constructed by approximating the Bellman optimality equation (Equation \ref{eqt:opt}) with empirical model. To be more specific, the empirical model is maintained and updated based on the visiting counts
\[ N^k_h(s,a):= \sum_{\tau=1}^{k-1} \bI\cbrk{(s^{\tau}_h,a^{\tau}_h)=(s,a)}, \ N^k_h(s,a,s'):= \sum_{\tau=1}^{k-1} \bI\cbrk{(s^{\tau}_h,a^{\tau}_h,s^{\tau}_{h+1})=(s,a,s')}. \]
The empirical model $\hat{P}^k_h$ for step $h$ in episode $k$ is set to be the visiting frequency
\[ \hat{P}^k_h(s'|s,a)=\frac{N^k_h(s,a,s')}{N^k_h(s,a)\vee1}.\]

The bonus term $b^k_h$ is composed of two factors: the estimation error of the next-state value distribution and the Lipschitz constant of the risk measure. The estimation error  can be bounded as
\[ \norm{(\hat{P}^k_h,V^*_{h+1})-(P_h,V^*_{h+1})}_{\infty} \leq \sqrt{\frac{\iota}{2(N^k_h(\cdot,\cdot)\vee 1)}}, \]
where $\iota$ is a confidence level to be specified later. This error term captures the uncertainty in the empirical model. The Lipschitz constant $L_{\infty,h}$ of the risk measure reflects its sensitivity to changes in  the underlying distributions. 
Multiplying these two factors together yields the exploration bonus term, which is added to the empirical value function estimate. This bonus term encourages exploration in situations where the model estimation error is large or the risk measure is sensitive. By carefully designing the bonus term, \texttt{UCBVI-DRM} achieves optimism in its value function estimates, promoting exploration while considering the nonlinearity of the risk measure. This allows the algorithm to balance exploration and exploitation, taking into account the uncertainty in the model and the smoothness of the risk measure.

\begin{algorithm}[tb]
	\caption{\texttt{UCBVI-DRM}}
	\label{alg:UCBVI-DRM}
	\begin{algorithmic}[1]
		\State{Input: $\prt{\rho_h}_{h\in[H-1]}$, $T$ and $\delta$}
		\State{Initialize $N^k_h(\cdot,\cdot)\leftarrow0$, $\hat{P}^k_h(\cdot,\cdot)\leftarrow \frac{1}{S}\textbf{1}$}
		\For{$k = 1:K$}
		\State{$V^k_{H}(\cdot)\leftarrow \max_a r_H(\cdot,a)$}
		\For{$h = H-1:1$}
		\State{$b^k_h(\cdot,\cdot)\leftarrow L_{\infty,h}\sqrt{\frac{\iota}{2(N^k_h(\cdot,\cdot)\vee 1)}}$}
		\State{$Q^k_h(\cdot,\cdot)\leftarrow r_h(\cdot,\cdot)+\rho_h(V^k_{h+1},\hat{P}^k_h(\cdot,\cdot))+b^k_h(\cdot,\cdot)$}
		\State{$V^k_h(\cdot)\leftarrow \max_{a\in\cA} Q^k_h(\cdot,a)$}
		\EndFor
		\State{Receive $s^k_1$}
		\For{$h = 1:H$}
		\State{Take action $a^k_h\leftarrow \arg\max_{a\in\cA} Q^k_h(s^k_h,a)$ and transit to $s^k_{h+1}$}
		\State{Update $N^k_h(\cdot,\cdot)$ and $\hat{P}^k_h(\cdot,\cdot)$}
		\EndFor
		\EndFor
	\end{algorithmic} 
\end{algorithm}

\begin{remark}
Algorithm \ref{alg:UCBVI-DRM} provides a general framework that subsumes other algorithms such as \texttt{ICVaR-VI} in \cite{duprovably} for dynamic CVaR and \texttt{OCE-VI} in \cite{xu2023regret} for dynamic OCE.  In particular, our bonus term matches theirs by setting $L_{\infty,h}=(H-h)/\alpha$ for CVaR and $L_{\infty,h}=u(-(H-h))$ for OCE. Therefore, Algorithm \ref{alg:UCBVI-DRM} generalizes these algorithms and provides a unified framework for addressing different risk measures. 
\end{remark}

\vspace{-2ex}
\subsection{\texttt{OVI-DRM}}
The \texttt{OVI-DRM} algorithm (see Algorithm \ref{alg:OVI-DRM}) is a model-based algorithm which   injects the optimism in the estimated model. It operates at a high level as follows.
For each step $h$ in episode $k$, the algorithm constructs an optimistically estimated transition model $\tilde{P}^k_h$ based on a high probability concentration bound on the empirical transition model $\hat{P}^k_h$. This optimistic model allows for exploration and promotes optimism in the face of uncertainty. Using the optimistic model $\tilde{P}^k_h$, the algorithm approximates the Bellman optimality equation (Equation \ref{eqt:opt}) to obtain optimistic value functions $Q^k_h$. 

The optimistic model $\tilde{P}^k_h$ is obtained by applying a subroutine called \texttt{OM} (Optimistic Model) that takes the empirical model $\hat{P}^k_h(s,a)$, the value at the next step $V^k_{h+1}$, and a confidence radius $c^k_h(s,a)$ as input. The subroutine constructs an optimistic model within an $\ell_1$ norm ball around $\hat{P}^k_h(s,a)$
\[ \norm{P_h(s,a)-\hat{P}^k_h(s,a)}_1 \leq c^k_h(s,a).\]

$c^k_h$ represents the confidence radius around the empirical model within which the true model lies with high probability. Due to the space limit, we defer  the details of the subroutine \texttt{OM} (Algorithm \ref{alg:OM}) in Appendix \ref{app:sub}. Note that \texttt{OM} is computationally efficient with complexity $\cO(S\log S)$. 

The optimism of the model induces optimism in the value estimates
\[ (\tilde{P}^k_h(s,a),V^k_{h+1}) \succeq  (P_h(s,a),V^k_{h+1}) \succeq  (P_h(s,a),V^*_{h+1}) \Longrightarrow Q^k_h(s,a) \ge Q^*_h(s,a). \]
By using an optimistically estimated model, the \texttt{OVI-DRM} algorithm promotes exploration and encourages the agent to take actions that have the potential for higher values.
\begin{algorithm}[tb]
	\caption{\texttt{OVI-DRM}}
	\label{alg:OVI-DRM}
	\begin{algorithmic}[1]
        \State{Input: $\prt{\rho_h}_{h\in[H-1]}$, $T$ and $\delta$}
		\State{Initialize $N^k_h(\cdot,\cdot)\leftarrow0$, $\hat{P}^k_h(\cdot,\cdot)\leftarrow \frac{1}{S}\textbf{1}$}
		\For{$k = 1:K$}
		\State{$V^k_{H}(\cdot)\leftarrow \max_a r_H(\cdot,a)$}
		\For{$h = H-1:1$}
		\State{$c^k_h(\cdot,\cdot)\leftarrow \sqrt{\frac{2S}{N^k_h(\cdot,\cdot)\vee 1}\iota}$}		\State{$\tilde{P}^k_h(\cdot,\cdot)\leftarrow \texttt{OM}\prt{\hat{P}^k_h(\cdot,\cdot),V^k_h,c^k_h(\cdot,\cdot)}$}
		\State{$Q^k_h(\cdot,\cdot)\leftarrow r_h(\cdot,\cdot)+\rho_h(V^k_{h+1},\tilde{P}^k_h(\cdot,\cdot))$}
		\State{$V^k_h(\cdot)\leftarrow \max_{a\in\cA} Q^k_h(\cdot,a)$}
		\EndFor
		\State{Receive $s^k_1$}
		\For{$h = 1:H$}
		\State{Take action $a^k_h\leftarrow \arg\max_{a\in\cA} Q^k_h(s^k_h,a)$ and transit to $s^k_{h+1}$}
		\State{Update $N^k_h(\cdot,\cdot)$ and $\hat{P}^k_h(\cdot,\cdot)$}
		\EndFor
		\EndFor
	\end{algorithmic} 
\end{algorithm}

\section{Main Results}
\label{sec:main}
To simplify notations, we define $\tilde{L}_{1,t}\triangleq\prod^{t}_{i=1} L_{1,i}$ for $t\in[H]$.
\subsection{Worst-case Regret Bounds}
\begin{theorem}[Worst-case regret upper bound]
\label{thm:worst_ub}
Fix $\delta\in(0,1)$. Suppose Assumption \ref{asp:rm} holds. Algorithm \ref{alg:UCBVI-DRM} and Algorithm \ref{alg:OVI-DRM} satisfies for any MDP $\cM$
\begin{align*}
\text{Regret}(K)
\leq\tilde{\mathcal{O}}\prt{\sum^{H-1}_{h=1}L_{\infty,h}\tilde{L}_{1,h-1}\cdot\sqrt{S^2AK}}
\end{align*}
with probability at least $1-\delta$, where $\iota\triangleq\log(4SAT/\delta)$.
\end{theorem}
The proof sketch of Theorem \ref{thm:worst_ub} is shown in Section \ref{sec:pf}.
\begin{remark}
In the risk-neutral setting, The Lipschitz constants take $L_{\infty,h}=H-h$ and $L_{1,h}=1$, which leads to the bound of $\tilde{\mathcal{O}}\prt{H^2\sqrt{S^2AK}}$.
\end{remark}

\begin{theorem}[Minimax  Lower Bound]
\label{thm:mini_lb}
 For any algorithm $\texttt{alg}$, there exists an MDP $\mathcal{M}$ such that for sufficiently large $K$ 
\[ \mathbb{E}[\text{Regret}(\texttt{alg},\mathcal{M},K)]\ge \Omega\prt{c_{\rho}H\sqrt{SAT}},\]
where $c_{\rho}$ is a constant dependent on the risk measure\footnote{For more details, please refer to Appendix \ref{app:lb}.}.
\end{theorem}
The construction of proof is based on \cite{domingues2021episodic}. 
\paragraph{Comparisons and Discussions.}
We compare our regret  bounds with that of \cite{xu2023regret} in the dynamic OCE setting.
By instantiating the Lipschitz constants of OCE with $L_{\infty,h}=u(-H+h)$ and $L_{1,h}=u'_{-}(-H+h)$, their bound can be translated into $\tilde{\mathcal{O}}\prt{\sum^{H-1}_{h=1}L_{\infty,h}\sqrt{\tilde{L}_{1,h-1}S^2AK}}$. Our bound matches their result with  additional factors $\sqrt{\tilde{L}_{1,h-1}}$. This is because they employ a change-of-measure technique based on the concave optimization representation of OCE to bound derive a tighter recursion of value gaps, which cannot be easily extended to general risk measures. The Algorithm \ref{alg:UCBVI-DRM}, however, still enjoy the same regret bound as that in \cite{xu2023regret} since our algorithm reduces to \texttt{OCE-VI} for the dynamic OCE. Furthermore, numerical experiments in Appendix \ref{app:lb} shows that Algorithm \ref{alg:OVI-DRM}  empirically outperforms Algorithm \ref{alg:UCBVI-DRM} for dynamic CVaR problem. Our lower bound also matches the lower bound in \cite{xu2023regret}, both of which are tight in $S,A,K,H$ and depend on some constant related to the risk measure.

In the dynamic CVaR setting, which is a special case of dyamic OCE, our upper bound matches the bound $\tilde{\mathcal{O}}\prt{H^2\sqrt{S^2AK}/\sqrt{\alpha^H}}$ in \cite{duprovably} up to a factor of $1/\sqrt{\alpha^H}$. Algorithm \ref{alg:UCBVI-DRM} subsumes \texttt{ICVaR-VI} for the dynamic CVaR problem. In contrast to \cite{duprovably} that provides a algorithm-dependent lower bound, we provide  minimax  and gap-dependent lower bound.
Furthermore,  \cite{lamrisk} considers dynamic coherent risk measures and non-linear function approximation, and their regret bounds are derived under the assumption of a weak simulator. As a result, the regret bounds provided in our work are not directly comparable to theirs, even if their results are specialized to the tabular MDP setting

Theorem \ref{thm:worst_ub} and Theorem \ref{thm:mini_lb} imply that RSRL with Lipschitz DRM can achieve regret bound that is minimax-optimal in terms of $K$ and $A$. Specifically, the gap between the upper and lower bounds is determined by two factors: $\sqrt{S}$ and a multiplicative Lipschitz constant $\tilde{L}_{1,H}$. we provide more technical details  behind this in the proof sketch. Achieving further improvements in these factors can be challenging, especially for general risk measures, under the mild Lipschitz assumption. 

\subsection{Gap-dependent Regret Bounds}
Fix $h\in[H]$, $(s,a)\in\cS\times\cA$, the \emph{sub-optimality gap} of $(s,a)$ at step $h$ is defined as 
$\Delta_h(s,a)\triangleq V^*_h(s,a)-Q^*_h(s,a)$. The \emph{minimum sub-optimality gap} is defined as the minimum non-zero gap 
\[ \Delta_{\text{min}} \triangleq \min_{h,s,a} \cbrk{ \Delta_h(s,a): \Delta_h(s,a)>0  }.\]
\begin{theorem}[Gap-dependent regret upper bound]
\label{thm:gap_ub}
Fix $\delta\in(0,1)$. With probability at least $1-\delta$, Algorithm \ref{alg:UCBVI-DRM} and Algorithm \ref{alg:OVI-DRM} satisfy
\begin{align*}
\text{Regret}(K)
&\leq  \mathcal{O}\prt{  \frac{ S^2A H\prt{ \sum_{h=1}^{H-1}  \tilde{L}_{1,h-1}L_{\infty,h}}^2 }{\Delta_{\text{min}}} \log(SAT)}.
\end{align*}
\end{theorem}
We follow the standard convention in the literature for the gap-dependent lower bound. The lower bound is stated for algorithms that
have sublinear worst-case regret. Specifically, we say that an algorithm $\texttt{alg}$ is  $\alpha$-uniformly good if for any MDP $\cM$, there exists a constant $C_{\cM}>0$ such that $\text{Regret}(\texttt{alg},\cM,K)\leq C_{\cM} K^{\alpha}$. The construction of proof is based on \cite{simchowitz2019non}. 
\begin{theorem}[Gap-dependent regret lower bound]
\label{thm:gap_lb}
There exists an MDP $\cM$ such that any $\alpha$-uniformly good algorithm $\texttt{alg}$ satisfies
    \begin{align*}
        \lim_{K\rightarrow\infty} \frac{\text{Regret}(\texttt{alg},\cM,K)}{\log K}= \Omega\prt{(1-\alpha)\sum_{(s,a):\Delta_1(s,a)>0} \frac{ (c_{\rho} H)^2}{\Delta_1(s,a)} }
    \end{align*}
\end{theorem}

To our knowledge, Theorem \ref{thm:gap_ub} provides the first result showing that RSRL with DRMs can achieve $\log T$-type regret. 
The lower bound shows that for sufficiently large $K$, the logarithmic dependence on $K$ is unavoidable. Notably, it  implies that our algorithms have a tight dependency on $A$ and $K$.
Furthermore, the presence of a constant factor in both the upper and lower bounds suggests that the specific choice of risk measure can significantly impact their performance.

\section{Proof Sketch of Theorem \ref{thm:worst_ub}}
\label{sec:pf}
For simplicity, we only provide the proof sketch for \texttt{UCBVI-DRM}. The proof structure  builds upon the framework established in \cite{azar2017minimax}, but we introduce new techniques to address the specific challenges posed by nonlinear risk measures: (i) the Lipschitz continuity w.r.t. $\norm{\cdot}_{\infty}$ together with the DKW inequality to ensure the optimism (step 1), and (ii) the Lipschitz continuity w.r.t. $\norm{\cdot}_{1}$ together with a transport inequality (see Lemma \ref{lem:tran}) to obtain recursions  of suboptimality gap.
\paragraph{Step 1: establish optimism.}
We first show that $V^k_1(s^k_1) \ge V^*_1(s^k_1), \forall k \in [K]$ with high probability. Using the  Lipschitz property of $\rho_h$ w.r.t. $\norm{\cdot}_{\infty}$ and the
DKW inequality (Fact \ref{fct:dkw}) 
\begin{align*}
    \rho_h\prt{V^*_{h+1},P_{h}(s,a)}-\rho_h\prt{V^*_{h+1},\hat{P}^k_{h}(s,a)}  &\leq L_{\infty,H-h}\norm{(V^*_{h+1},P_{h}(s,a))- \prt{ V^*_{h+1},\hat{P}^k_{h}(s,a) } }_{\infty}\\
    &\leq L_{\infty,H-h}\sqrt{\frac{\iota}{2(N^k_h(s,a)\vee 1)}}=b^k_h(s,a).
\end{align*}
The results follows from the monotonicity of $\rho_h$ and induction.

\paragraph{Step 2: regret decomposition.}
We define $\Delta^k_h\triangleq V^k_h-V_h^{\pi^k} \in \bR^{\cS}$ 
and $\delta^k_h\triangleq\Delta^k_h(s^k_h)$.  The optimism implies that the regret can be bounded by the surrogate regret
\begin{align*}
\text{Regret}(K)&=\sum^K_{k=1}V^{*}_1(s^k_1)-V^{\pi^k}_1(s^k_1) \leq \sum^K_{k=1}V^{k}_1(s^k_1)-V^{\pi^k}_1(s^k_1)=\sum^K_{k=1}\delta^k_1.
\end{align*}
We write $r^k_h\triangleq r_h(s^k_h,a^k_h)$, $b^k_h\triangleq b^k_h(s^k_h,a^k_h)$, $N^k_h\triangleq N^k_h(s^k_h,a^k_h)$, $\hat{P}^k_h(s^k_h)\triangleq \hat{P}^k_h(s^k_h,a^k_h)$, and $P^{\pi^k}_h\triangleq P_h(s^k_h,a^k_h)$ for simplicity
. For any $h\in[H-1]$, We decompose $\delta^k_h$ as follows
\begin{align*}
    \delta^k_h
    &= b^k_h + \underbrace{ \rho_h\prt{V^k_{h+1},\hat{P}^k_h}-\rho_h\prt{V^k_{h+1},P^{\pi^k}_h} }_{(a)} + \underbrace{ \rho_h\prt{V^k_{h+1},P^{\pi^k}_h}-\rho_h\prt{V^{\pi^k}_{h+1},P^{\pi^k}_h} }_{(b)}.
\end{align*}
Bounding term (a) and term (b) requires new techniques compared with risk-neutral setting. To deal with the nonlinearity, we relate the two value difference terms to the distances between value distribution via Lipschitzness of the risk measure.
To bound term (a), we use the  Lipschitz property  w.r.t. $\norm{\cdot}_{\infty}$ to get
\begin{align*}
(a) &\leq  L_{\infty,h}\norm{ \prt{V^k_{h+1},\hat{P}^k_h} - \prt{V^k_{h+1},P^{\pi^k}_h} }_{\infty} \leq L_{\infty,h} \norm{\hat{P}^k_h-P^{\pi^k}_h}_1 \leq L_{\infty,h} \cdot c^k_h.
\end{align*}
Due to the linearity of expectation, standard analysis  bound term (b) by the sum of the vale gap at next step $\delta^k_{h+1}$ and a martingale noise, and then the recursion of the value gap is obtained. The derivation of a recursion in the presence of nonlinear $\rho_h$, however, leads to the main technical challenge. 
Our key observation to overcome the difficulty is a simple transport inequality in the following.
\begin{lemma}
\label{lem:tran}
$\norm{(x,P)-(y,P)}_{1}\leq \sum_{i=1}^n P_i\abs{x_i-y_i}.$
\end{lemma}
Together with the Lipschitz property  w.r.t. $\norm{\cdot}_{1}$, we have
\begin{align*}
    (b)&\leq L_{1,h}\norm{ \prt{V^k_{h+1},P^{\pi^k}_h} - \prt{V^{\pi^k}_{h+1},P^{\pi^k}_h} }_{1}
    \leq L_{1,h} \sum_{s^{\prime}\in\cS} P^{\pi^k}_h(s^{\prime})\abs{V^k_{h+1}(s^{\prime})-V^{\pi^k}_{h+1}(s^{\prime})}\\
    &= L_{1,h} \sum_{s^{\prime}\in\cS} P^{\pi^k}_h(s^{\prime})\prt{V^k_{h+1}(s^{\prime})-V^{\pi^k}_{h+1}(s^{\prime})}= L_{1,h} \cdot P^{\pi^k}_h \Delta^k_{h+1}\triangleq L_{1,h} (\epsilon^k_h + \delta^k_{h+1}),
\end{align*}
where $\epsilon^k_h\triangleq P^{\pi^k}_h \Delta^k_{h+1} -\Delta^k_{h+1}(s^k_{h+1})$ is a  martingale difference sequence, 
and the first equality is due to  $V^k_{h+1}(s^{\prime})\ge V^*_{h+1}(s^{\prime})\ge V^{\pi^k}_{h+1}(s^{\prime})$ for all $s^{\prime}$. Now we can bound $\delta^k_h$ recursively
\begin{align*}
    &\delta^k_h \leq L_{\infty,h}\cdot c^k_h+ L_{1,h}(\epsilon^k_h + \delta^k_{h+1}) 
    +b^k_h\leq 2L_{\infty,h}\cdot c^k_h+ L_{1,h}(\epsilon^k_h + \delta^k_{h+1}) 
    +c^k_h.
\end{align*}
\paragraph{Step 3: putting together.}
Unrolling the recursion and summing up over $K$ episodes yields
\begin{align*}
    \text{Regret}(K) \leq \sum_{k\in[K]} \delta^k_1 \leq 2\sum^{H-1}_{h=1}L_{\infty,h} \tilde{L}_{1,h-1}\sum_{k=1}^K c^k_{h} + \sum_{k=1}^K\sum^{H-1}_{h=1} \tilde{L}_{1,h} \epsilon^k_{h}.
\end{align*}
We bound the first term via a pigeonhole argument and bound the second term by the Azuma-Hoeffding inequality. The finial result follows from  a union bound.


\section{Conclusions}
\label{sec:con}
We propose two model-based algorithms for the broad class of Lipschitz DRMs. To establish the efficacy of our algorithms, we provide theoretical guarantees in the form of worst-case and gap-dependent regret upper bounds. To complement our upper bounds, we also establish regret lower bounds. These lower bounds  demonstrate the inherent difficulty of the problem.

There are several promising future directions.  It might be possible to improve the regret upper bounds by designing new algorithms or improving the analysis. Currently, our algorithms and analysis are primarily focused on tabular MDPs. However, extending the results to the setting of function approximation, such as linear function approximation, is an important and challenging task. The nonlinearity of risk measures poses a significant obstacle in this context. One potential approach to address this issue is to leverage techniques like value-targeted regression, as proposed in \cite{ayoub2020model,jia2020model}, and integrate them into our framework.

\bibliography{neurips_2023}
\bibliographystyle{apalike}


\newpage

\appendix

\section{Experiments}
\label{app:exp}
In this section, we provide some numerical results to validate the empirical performance of our algorithms. We compare our algorithms to the  algorithms \texttt{UCBVI} \cite{azar2017minimax} for risk-neutral RL and \texttt{RSVI2} \cite{fei2021exponential}  for RSRL with ERM.

\begin{figure}[ht]
\begin{center}
\centering
\begin{subfigure}{0.49\textwidth}
    \includegraphics[width=\textwidth]{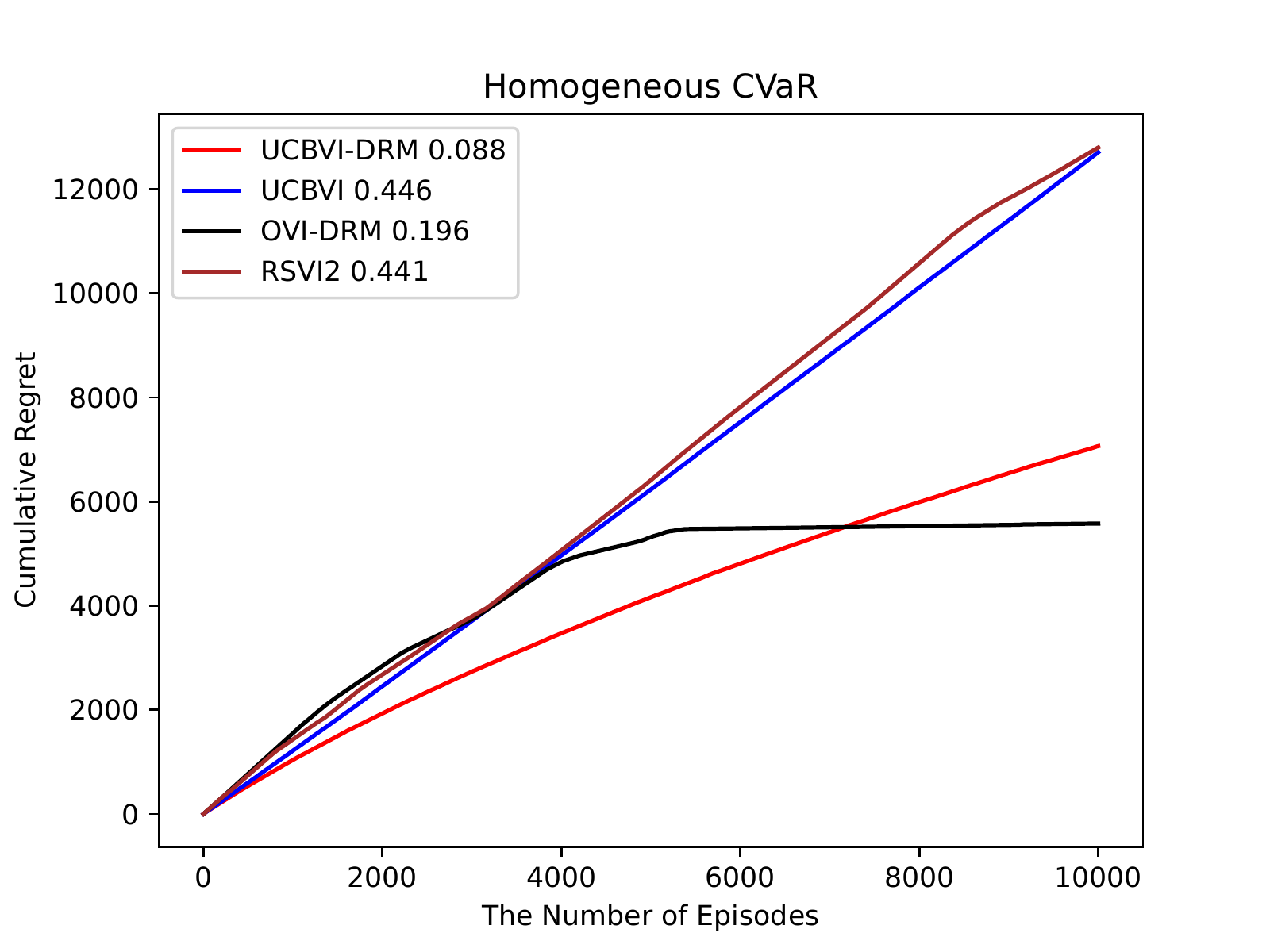}
    \caption{Comparison of different algorithms for homogeneous CVaR.}
    \label{fig:homo}
\end{subfigure}
\begin{subfigure}{0.49\textwidth}
    \includegraphics[width=\textwidth]{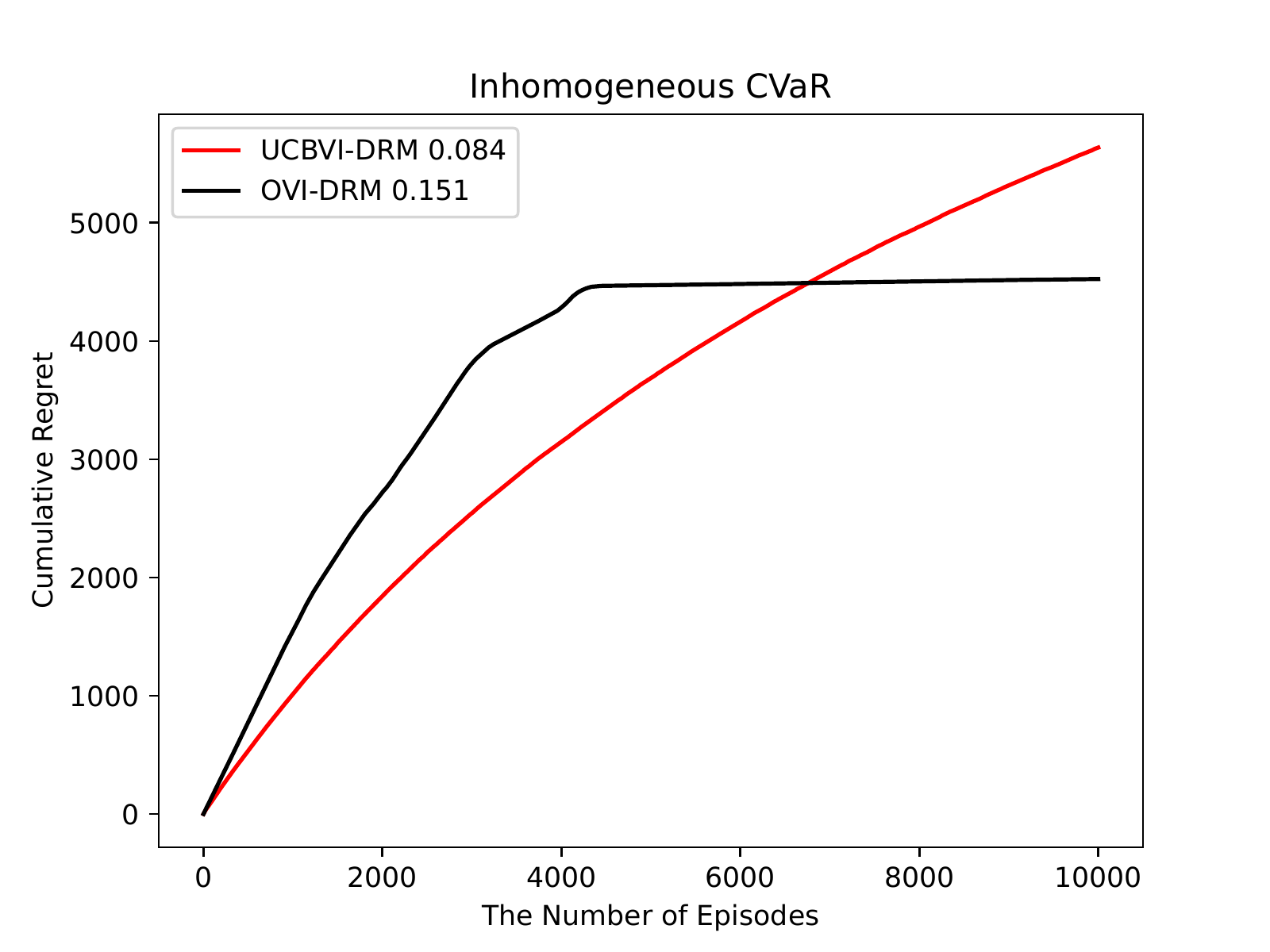}
    \caption{Comparison of different algorithms for in-homogeneous CVaR.}
    \label{fig:inhomo}
\end{subfigure}
\end{center}
\end{figure}

In our experiments, we focus on an  MDP with $S=3$ states, $A$ actions, and horizon $H$, which is similar to the construction  in \cite{duprovably}. The major difference is that we consider a non-stationary MDP. The MDP consists of a fixed initial state denoted as state 0, and three additional states denoted as states 1, 2, and 3. In step $2\leq h\leq H$, the three states generate reward 1,0 and 0.4, respectively. The agent starts from state $0$ in the first step, takes action from $[A]$, and then transitions to one of three states $\{1,2,3\}$ in the next step. Any action in $[A-1]$ leads to a uniform transition to state $1$ and state 2. The optimal action $A$ leads to  a transition to state 2 and state 3 with probability 0.001 and 0.999.

We consider the dynamic CVaR with the homogeneous and in-homogeneous setting. For the homogeneous CVaR, the coefficients at all steps are identical, while for the in-homogeneous CVaR, the coefficients at different steps could be different. We set  $\delta=0.005$, $A=5$, $H=5$ and $K=10000$. We set $\alpha=(0.05,0.05,0.05,0.05)$ and $\alpha=(0.09,0.08,0.07,0.05)$ for the homogeneous and in-homogeneous CVaR, respectively. We perform 5 independent runs for each algorithm.

As shown in Figure \ref{fig:homo}, \texttt{OVI-DRM} and \texttt{UCBVI-DRM} enjoy sublinear regret while the risk-neutral algorithm \texttt{UCBVI} and \texttt{RSVI2}  suffer linear regret. In particular, \texttt{OVI-DRM}
outperforms \texttt{UCBVI-DRM} because it achieves a better balance between exploration and exploitation. Figure \ref{fig:inhomo} only plots the results for our algorithms. It shows that \texttt{UCBVI-DRM} can also  learn  the optimal policy in the in-homogeneous CVaR setting. 

\section{Risk Measures}
\label{app:rm}
\subsection{Definitions}
\paragraph{Conditional Value at Risk (CVaR)}
The CVaR value \cite{rockafellar2000optimization} at level $\alpha\in(0,1)$ for a distribution $F$ is defined as 
\[  C_{\alpha}(F)\triangleq \sup_{\nu\in\bR}\cbrk{ \nu - \frac{1}{\alpha} \bE_{X\sim F}[(\nu-X)^+]}.    \]
\cite{acerbi2002coherence} showed that when $F$ is a continuous distribution,  $C_{\alpha}(F)=\bE_{X\sim F}[X|X\leq F^{-1}(\alpha)]$.
\paragraph{Spectral risk measure (SRM)}
SRM is class of risk measures that  generalizes CVaR via adopting a non-constant weighting function over the quantiles \cite{acerbi2002spectral}. The SRM value of $F$ is defined as
\[  S_{\phi}(F)\triangleq \int_0^1 \phi(y)F^{-1}(y)dy, \]
where $\phi:[0,1]\rightarrow[0,\infty)$ is  the weighting function. \cite{acerbi2002spectral} showed that an SRM is coherent  if $\phi$ is is decreasing and satisfies that $\int_0^1\phi(y)dy=1$. SRM can be viewed as a weighted average of the quantiles, with weight specified by $\phi(y)$. In fact, $S_{\phi}(F)$ specializes in $C_{\alpha}(F)$ for $\phi(y)=\frac{1}{\alpha}\mathbb{I}\{0\leq y\leq \alpha\}$.
\paragraph{Distortion risk measure}
For a distribution $F\in \sD([0,\infty))$, the distortion risk measure \cite{balbas2009properties,wirch2001distortion} $\rho_g(F)$ is defined as
\[  \rho_g(F)\triangleq \int_0^{\infty}g(1-F(x))dx,    \]
where $g:[0,1]\rightarrow[0,1]$ is a continuous increasing function with $g(0)=0$ and $g(1)=1$. We refer to $g$ as the distortion function. Distortion risk measure is coherent if and only if $g$ is convex. Similar to SRM, distortion risk measure can also recover CVaR by choosing proper $g$.
\paragraph{Entropic risk measure (ERM)}
ERM adjusts the risk attitude of the user through the exponential utility function. In particular, the ERM value of $F$ with coefficient $\beta\neq 0$ is defined as
\[  U_{\beta}(F)\triangleq \frac{1}{\beta}\log(\bE_{X\sim F}[\exp(\beta X)])=\frac{1}{\beta}\log\prt{\int_{\bR}\exp(\beta x)dF(x)}.\]
Notably, ERM is the prime example of a convex risk measure which is not coherent \cite{rudloff2008entropic}.
\paragraph{Optimized certainty equivalent (OCE)}
The OCE \cite{ben2007old} value of $F$ associated with a utility function $u$ is given by
\[  C_u(F)\triangleq \sup_{\lambda} \cbrk{\lambda+\bE_{X\sim F}[u(X-\lambda)]}=\sup_{\lambda} \cbrk{\lambda+\int_{\bR}u(x-\lambda)dF(x)}, \]
where $u$ is a non-decreasing, closed utility function that satisfies $u(0)=0$ and $1\in \partial u(0)$. THE OCE is risk-averse (risk-seeking) if and only $u$ is concave (convex). OCE subsumes important examples of popular risk measures, including the ERM and CVaR.
\subsection{Lipschitz Property}
We summarizes the Lipschitz constants of common risk measures over a finite interval $[a,b]$ in Table \ref{tab:lp}. \cite{prashanth2022wasserstein} provides the Lipschitz constants of SRM, OCE, and distortion risk measure w.r.t. the Wasserstein distance or $\norm{\cdot}_1$. \cite{huang2021off} provides the Lipschitz constants of distortion risk measure w.r.t. $\norm{\cdot}_{\infty}$. 

For completeness, we will derive the  Lipschitz constants of SRM and OCE w.r.t. $\norm{\cdot}_{\infty}$ in the following. Fact \ref{fct:1_inf} offers a simple way to derive the Lipschitz constant of a risk measure w.r.t. $\norm{\cdot}_1$ based on that w.r.t. the $\norm{\cdot}_{\infty}$. Therefore, the Lipschitz constants of SRM w.r.t. $\norm{\cdot}_{\infty}$ can take $L_{\infty}(M_{\phi},[a,b])=(b-a) \cdot L_{1}(M_{\phi},[a,b])=(b-a)\max |\phi(x)|$. As a special case, we have $L_{\infty}(C_{\alpha},[0,M])=\frac{b-a}{\alpha}$.
\begin{fact}
\label{fct:1_inf}
If a functional $\rho$ has Lipschitz constant $L_{1}([a,b])$ over $\sD([a,b])$, then it has  Lipschitz constant $L_{\infty}([a,b])=L_{1}([a,b])(b-a)$.
\end{fact}
\begin{proof}
    Suppose $\rho$ has Lipschitz constant $L_{1}([a,b])$ over $\sD([a,b])$, then 
    \[|\rho(F)-\rho(G)|\leq L_{1}([a,b])\norm{F-G}_1 \leq L_{1}([a,b])\norm{F-G}_{\infty}(b-a), \forall F,G\in \sD([a,b]).\]
    This implies that  $L_{\infty}([a,b])=L_{1}([a,b])(b-a)$ is a valid choice.
\end{proof}
\begin{fact}
The Lipschitz constants of OCE w.r.t. $\norm{\cdot}_{\infty}$ is $L_{\infty}(C_{u},[a,b])=-u(a-b)$ for concave utility function and $L_{\infty}(C_{u},[a,b])=u(b-a)$ for convex utility function.
\end{fact}
\begin{proof}
    Let $\lambda_1,\lambda_2\in[a,b]$ satisfy 
    \begin{align*}
        C_u(F)&= \lambda_1 + \int_a^b u(x-\lambda_1) dF(x) =\max_{\lambda\in[a,b]} \lambda + \int_a^b u(x-\lambda) dF(x) \\
        C_u(G)&= \lambda_2 + \int_a^b u(x-\lambda_2) dG(x) =\max_{\lambda\in[a,b]} \lambda + \int_a^b u(x-\lambda) dG(x).
    \end{align*}
    Without loss generality, we assume $C_u(F)>C_u(G)$. It holds that
    \begin{align*}
        C_u(F)-C_u(G) &= \lambda_1 + \int_a^b u(x-\lambda_1) dF(x)-\lambda_2 - \int_a^b u(x-\lambda_2) dG(x) \\
        &\leq \lambda_1 + \int_a^b u(x-\lambda_1) dF(x)-\lambda_1 - \int_a^b u(x-\lambda_1) dG(x) \\
        &=\int_a^b u(x-\lambda_1) dF(x) -  \int_a^b u(x-\lambda_1) dG(x) \\
        &= u(x-\lambda_1)F(x)|_a^b-\int_a^b F(x) du(x-\lambda_1)-u(x-\lambda_1)G(x)|_a^b+\int_a^b G(x) du(x-\lambda_1) \\
        &= \int_a^b (G(x)-F(x)) du(x-\lambda_1) \\
        &\leq \int_a^b  du(x-\lambda_1) \cdot \norm{F-G}_{\infty} \\
        &= (u(b-\lambda_1)-u(a-\lambda_1))\norm{F-G}_{\infty} \\
        &\leq \max_{\lambda\in[a,b]}(u(b-\lambda)-u(a-\lambda))\norm{F-G}_{\infty}=L_{\infty}(C_u,[a,b])\norm{F-G}_{\infty},
    \end{align*}
    where the second inequality is due to that $u$ is non-decreasing. For concave utility function, we can bound the last term as 
    \[\max_{\lambda\in[a,b]}(u(b-\lambda)-u(a-\lambda)) = u(b-b)-u(a-b)=-u(a-b).\]
    For convex utility function, we can bound the last term as 
    \[\max_{\lambda\in[a,b]}(u(b-\lambda)-u(a-\lambda)) = u(b-a)-u(a-a)=u(b-a).\]
\end{proof}
\begin{fact}
The Lipschitz constants of OCE w.r.t. $\norm{\cdot}_{1}$ is $L_{1}(C_{u},[a,b])=u'(a-b)$ for concave utility function and $L_{1}(C_{u},[a,b])=u'(b-a)$ for convex utility function.
\end{fact} 
\begin{proof}
    Let $\lambda_1,\lambda_2\in[a,b]$ satisfy 
    \begin{align*}
        C_u(F)&= \lambda_1 + \int_a^b u(x-\lambda_1) dF(x) =\max_{\lambda\in[a,b]} \lambda + \int_a^b u(x-\lambda) dF(x) \\
        C_u(G)&= \lambda_2 + \int_a^b u(x-\lambda_2) dG(x) =\max_{\lambda\in[a,b]} \lambda + \int_a^b u(x-\lambda) dG(x).
    \end{align*}
    Without loss generality, we assume $C_u(F)>C_u(G)$. It holds that
    \begin{align*}
        C_u(F)-C_u(G) &= \lambda_1 + \int_a^b u(x-\lambda_1) dF(x)-\lambda_2 - \int_a^b u(x-\lambda_2) dG(x) \\
        &\leq \int_a^b u(x-\lambda_1) dF(x) -  \int_a^b u(x-\lambda_1) dG(x) \\
        &= \int_a^b (G(x)-F(x)) du(x-\lambda_1) \\
        &= \int_a^b (G(x)-F(x)) u'(x-\lambda_1)dx \\
        &\leq \max_{\lambda\in[a,b],x\in[a,b]}u'(x-\lambda)\int_a^b |G(x)-F(x)| dx \\
        &=L_{1}(C_u,[a,b])\norm{F-G}_{1},
    \end{align*}
    where the second inequality is due to the non-negativity of $u'$. For concave utility function, we can bound the last term as 
    \[\max_{\lambda\in[a,b],x\in[a,b]}u'(x-\lambda) = u'(a-b).\]
    For convex utility function, we can bound the last term as 
    \[\max_{\lambda\in[a,b],x\in[a,b]}u'(x-\lambda) = u'(b-a).\]
\end{proof}
OCE subsumes ERM when $u(x)=\frac{\exp(\beta x)-1}{\beta}$. In particular, $L_{\infty}(U_{\beta},[a,b])=-u(a-b)=-\frac{\exp(\beta (a-b))-1}{\beta}=\frac{\exp(|\beta| (b-a))-1}{|\beta|}$ for concave utility ($\beta<0$) and $L_{\infty}(U_{\beta},[a,b])=u(b-a)=\frac{\exp(\beta (b-a))-1}{\beta}$ for convex utility ($\beta>0$). 
\begin{table}
	\caption{Lipschitz constants of common risk measures}	
	\label{tab:lp}
	\centering
	\begin{tabular}{ccccc}
		\hline
		Lipschitz constant & CVaR &SRM &distortion risk measure  &OCE\footnote{For convenience, we only list OCE with concave risk measure.} \\ [0.5ex] 
		\hline
		$L_1([a,b])$ & $\frac{1}{\alpha}$   &$\max \phi(x)$  &$\max g'(x)$ & $ u'(a-b)$ \\ [0.5ex] 
		$L_{\infty}([a,b])$ & $\frac{b-a}{\alpha}$  &$(b-a)\max \phi(x)$   &$(b-a)\max g'(x)$  & $-u(a-b)$\\
		\hline
	\end{tabular}
\end{table} 
\section{Subroutine}
\label{app:sub} 
We present the subroutine \texttt{OM} used in Algorithm \ref{alg:OVI-DRM} in this section. Fix an $(s,a,k,k)$, \texttt{OM} takes the empirical model $\hat{P}^k_h(s,a)$, the value at the next step $V^k_{h+1}$, and a confidence radius $c^k_h(s,a)$ as input and outputs the optimistic model $\tilde{P}^k_h(s,a)$. For a PMF $P$ and a real number $c>0$, denote by $B_1(P,c)\triangleq \cbrk{P'|\norm{P'-P}_1\leq c}$ the $\ell_1$ norm ball centered at $P$ with radius $c$. 

\begin{algorithm}[tb]
	\caption{\texttt{OM}}
	\label{alg:OM}
	\begin{algorithmic}[1]
        \State{Input: $P=(P(s_{1}),\dots,P(s_{S}))$, $V=(V(s_{1}),\dots,V(s_{S}))$ and $c>0$}
		\State{Sorting: let $V'=(V(s_{(1)}),\cdots,V(s_{(S)}))$ such that $V(s_{(1)})\leq V(s_{(2)}) \leq \cdots \leq V(s_{(S)})$}
        \State{Permutation: let $P'=(P(s_{(1)}),\cdots,P(s_{(S)}))$} 
        \State{Transport: sequentially move probability mass $\frac{c}{2}\land1$ of the leftmost states to $s_{(S)}$} 
	\end{algorithmic} 
\end{algorithm}
Recall that we $F\succeq G$ denotes that $F(x)\leq G(x),\forall x\in\bR$. Lemma \ref{lem:om} shows that 
\texttt{OM} can output an optimistic model $\tilde{P}$ whose value distribution dominates those generated by the model within the concentration ball. 
\begin{lemma}
\label{lem:om}    
Let $P,V,c$ be the input of \texttt{OM} and $\tilde{P}$ be the output. It holds that 
\[ (\tilde{P},V) \succeq (Q,V), \forall Q \in B_1(P,c). \]
\end{lemma}
\begin{proof}
For simplicity, let $P=(P_1,\cdots,P_n)$, $V\in\bR^n$ satisfying $V_1\leq V_2\cdots \leq V_n$. Observe that the CDF $(P,V)$ is a piecewise constant function. Hence it suffices to show that 
    \[ \sum_{j=1}^i \tilde{P}_j \leq  \sum_{j=1}^i Q_j, \forall i\in[n], \forall Q \in B_1(P,c). \]
Let $l\triangleq \min\cbrk{i|\sum_{j=1}^i P_j\ge\frac{c}{2}}$. There are two cases. 

Case 1: $P_n +\frac{c}{2}\leq1$. Since $\sum_{j=1}^{l-1} P_{j}<\frac{c}{2}$ and $\sum_{j=1}^{l} P_{j}\ge \frac{c}{2}$, we have 
\[ \tilde{P}_i =\begin{cases}
	0, i\in[l-1] \\
	\sum_{j=1}^l P_j-\frac{c}{2}, i=l \\
	P_i, l+1 \leq i \leq n-1 \\
    P_n+\frac{c}{2}, i=n
	\end{cases}\]
and thus
\[ \sum_{j=1}^{i} \tilde{P}_{j} =\begin{cases}
	0, i\in[l-1] \\
	\sum_{j=1}^l P_j-\frac{c}{2}, l\leq i\leq n-1 \\
    1, i=n
	\end{cases}\]
For any $Q\in B_1(P,c)$, it holds that
\[ \sum_{j=1}^i \tilde{P}_j \leq  \sum_{j=1}^i Q_j, \forall i\in[n]. \]
Otherwise $\sum_{j=1}^k Q_j < \sum_{j=1}^k P_j-\frac{c}{2}$ for some $l \leq k \leq n-1$, which implies $\sum_{j=k+1}^n Q_j =1-\sum_{j=1}^k Q_j> 1-\sum_{j=1}^k P_j+\frac{c}{2}=\sum_{j=1}^k P_j+\frac{c}{2}$. This leads to a contradiction $\norm{P-Q}_1=\sum_{j\in[n]}|P_j-Q_j| \ge|\sum_{j=1}^k P_j-\sum_{j=1}^k Q_j|+|\sum_{j=k+1}^n P_j-\sum_{j=k+1}^n Q_j|>c$.

Case 2: $P_n +\frac{c}{2}>1$. In this case, $\tilde{P}_j=0$ for $j\in[n-1]$ and $\tilde{P}_n=1$. It is obvious that 
\[ \sum_{j=1}^i \tilde{P}_j \leq  \sum_{j=1}^i Q_j, \forall i\in[n], \forall Q \in B_1(P,c). \]
Therefore, we have
\[(\tilde{P},V) \succeq (Q,V), \forall Q \in B_1(P,c).\]
\end{proof}
Lemma \ref{lem:om} together with the monotonicity of $\rho_h$ implies that the output $\tilde{P}^k_h(s,a)$ satisfies  
\[  \rho_h(\tilde{P}^k_h(s,a),V^k_{h+1}) \ge \rho_h(P',V^k_{h+1}), \ \forall P'\in B_1(\hat{P}^k_h(s,a),c^k_h(s,a)).\]
In Appendix \ref{app:ub}, we will show that $P_h(s,a)\in B_1(\hat{P}^k_h(s,a),c^k_h(s,a))$ with high probability. Suppose $V^k_{h+1}(s)\ge V^*_{h+1}(s), \forall s$. It follows that 
\begin{align*}
    Q^k_h(s,a)&=r_h(s,a)+\rho_h(\tilde{P}^k_h(s,a),V^k_{h+1}) \ge  r_h(s,a)+\rho_h(P_h(s,a),V^k_{h+1}) \\
    &\ge r_h(s,a)+\rho_h(P_h(s,a),V^*_{h+1}) =Q^*_h(s,a), \forall (s,a).
\end{align*}
The second inequality is due to the monotonicity of $\rho_h$ together with the fact that
\[ V^k_{h+1} \ge V^*_{h+1} \Longrightarrow (P_h(s,a),V^k_{h+1}) \succeq  (P_h(s,a),V^*_{h+1}). \]
Then we have $V^k_h(s)=\max Q^k_h(s,a) \ge \max Q^*_h(s,a) =V^*_h(s), \forall s$. By induction, we obtain 
\[  V^k_h(s) \ge V^*_h(s), \forall (k,h,s). \]
This implies that the value functions induced by the optimistic models are indeed optimistic compared to the optimal value functions.  The computational complexity of \texttt{OM} is $O(S\log(S))$, since the computational complexity of each step is $O(S\log(S))$, $O(S)$, and $O(S)$.

\section{Proofs of Regret Upper Bounds}
\label{app:ub}
\subsection{Worst-case Regret Upper Bound}

We first prove the worst-case regret upper bound for Algorithm \ref{alg:UCBVI-DRM}.
\subsubsection{Proof for Algorithm \ref{alg:UCBVI-DRM}}
\paragraph{Step 1: verify optimism.}
Fix an arbitrary $\delta\in(0,1)$. Define the good event $\mathcal{G}_{\delta}$ as
{\small
\begin{align*}
  \mathcal{G}_{\delta}\triangleq &\left\{\norm{\hat{P}^k_h(\cdot|s,a)-P_h(\cdot|s,a)}_{1} \leq\sqrt{\frac{2S}{N^k_h(s,a)\vee 1}\iota}, \forall (s,a,k,h)\in\mathcal{S}\times\mathcal{A}\times[K]\times[H]\right\},  
\end{align*}
}
under which the empirical model concentrates around the true model under $\norm{\cdot}_1$. 
\begin{lemma}[High probability good event]
\label{lem:main_event}
The event $\mathcal{G}_{\delta}$ holds with probability at least $1-\delta$.
\end{lemma}
\begin{fact}[$\ell_1$ concentration bound, \cite{weissman2003inequalities}]
\label{fct:l1_con}
 Let $P$ be a probability distribution over a finite discrete measurable space $(\mathcal{X}, \Sigma) .$ Let $\widehat{P}_{n}$ be the empirical distribution of $P$ estimated from $n$ samples. Then with probability at least $1-\delta$,
$$
\left\|\widehat{P}_{n}-P\right\|_{1} \leq \sqrt{\frac{2|\mathcal{X}|}{n} \log \frac{1}{\delta}}.
$$
\end{fact}
Lemma \ref{lem:main_event} does not directly follow from a union bound together with Fact \ref{fct:l1_con} since the case $N^k_h(s,a)=0$ need to be checked.
\begin{proof}
	Fix some $(s,a,k,h)\in\mathcal{S}\times\mathcal{A}\times[K]\times[H]$. If $N^k_h(s,a)=0$, then we have $\hat{P}^k_h(\cdot|s,a)=\frac{1}{S}\textbf{1}$. A simple calculation yields that for  any $P_h(\cdot|s,a)$
	\begin{align*}
     \norm{\frac{1}{S}\textbf{1}-P_h(\cdot|s,a)}_{1}\leq2\leq\sqrt{2S\log(1/\delta)}.
	\end{align*}
	It follows that 
	\begin{align*}
	\mathbb{P} &\left(\norm{\hat{P}^k_h(\cdot|s,a)-P_h(\cdot|s,a)}_{1}\leq \left. \sqrt{\frac{2S}{N^k_h(s,a)\vee 1}\log(1/\delta)}\right| N^k_h(s,a)=0 \right) =1>1-\delta. 
	\end{align*}
	The event is true for the unseen state-action pairs. Now we consider the case that $N^k_h(s,a)>0$. By Fact \ref{fct:l1_con} , we have that for any integer $n\ge1$
	\begin{align*}
	\mathbb{P} & \left(\norm{\hat{P}^k_h(\cdot|s,a)-P_h(\cdot|s,a)}_{1}  \left.\leq\sqrt{\frac{2S}{N^k_h(s,a)}\log(1/\delta)} \right| N^k_h(s,a)=n \right)\ge1-\delta.
	\end{align*}
	Thus we have
	\begin{align*}
	&\mathbb{P} \left(   \norm{\hat{P}^k_h(\cdot|s,a)-P_h(\cdot|s,a)}_{1}\leq\sqrt{\frac{2S\log(1/\delta)}{N^k_h(s,a)}} \right)\\
	&=\sum_{n=0,1,\cdots}\mathbb{P} \left(   \norm{\hat{P}^k_h(\cdot|s,a)-P_h(\cdot|s,a)}_{1} \leq \left.\sqrt{\frac{2S\log(1/\delta)}{N^k_h(s,a)\vee 1}} \right| N^k_h(s,a)=n\right)\mathbb{P}(N^k_h(s,a)=n ) \\
	&\ge (1-\delta) \sum_{n=0,1,\cdots}\mathbb{P}(N^k_h(s,a)=n )=1-\delta.
	\end{align*}
	Applying a union bound over all $(s,a,k,h)$ and rescaling $\delta$ leads to the result.
\end{proof}
\begin{lemma}[Range of $V^*$]
\label{lem:range}
    For any MDP, it holds that $V^*_h(s)\in[0,H+1-h]$ for all $(s,h)\in\cS\times[H+1]$.
\end{lemma}
\begin{proof}
The proof follows from induction and Assumption \ref{asp:rm}. Observe that $V^*_{H+1}=0$. Suppose $V^*_{h+1}(s)\in[0,H-h]$ for any $s$, then we have
\begin{align*}
    0\leq Q^*_h(s,a)=r_h(s,a)+\rho_h(V^*_{h+1},P_h(s,a))\leq 1+H-h,
\end{align*}
where the inequalities are due to the Assumption \ref{asp:rm}. Then we have $V^*_h(s)=\max_a Q^*_h(s,a)\in[0,H+1-h]$. The induction is completed.
\end{proof}
\begin{fact}[DKW inequality for discrete distribution]
\label{fct:dkw}
    Let $\hat{P_n}$ be the empirical PMF for $(x,P)$ with  $n$ samples, then w.p. at least $1-\delta$
    \[ \norm{(x,P)-(x,\hat{P}_n)}_{\infty} \leq \sqrt{\frac{\log(2/\delta)}{2n}}. \]
\end{fact}
We remark that we can also derive a bound by Fact \ref{fct:inf_1}: w.p. $1-\delta$
\[ \norm{(x,P)-(x,\hat{P}_n)}_{\infty} \leq \norm{P-\hat{P}_n}_1 \leq \sqrt{\frac{2m\log(2/\delta)}{n}}. \]
However, this bound is looser than that from Fact \ref{fct:dkw} with a factor of $\sqrt{m}$.
\begin{lemma}[Optimistic value function]
\label{lem:main_opt} 
Conditioned on event $\mathcal{G}_{\delta}$, the sequence  $\{V^k_1(s^k_1)\}_{k\in[K]}$ produced by Algorithm \ref{alg:UCBVI-DRM} satisfies $V^k_1(s^k_1) \ge V^*_1(s^k_1), \forall k \in [K]$.
\end{lemma}
\begin{proof}
The proof follows from induction. Fix $k\in[K]$. It is evident that the inequality holds when $h=H+1$. Suppose the inequality holds for $h+1$. It follows that for any $(s,a)$
\begin{align*}
    Q^k_h(s,a)&=r_h(s,a)+\rho_h \prt{ V^k_{h+1},\hat{P}^k_{h}(s,a) } +b^k_h(s,a)\\
    &\ge r_h(s,a)+\rho_h \prt{ V^*_{h+1},\hat{P}^k_{h}(s,a) } +b^k_h(s,a)  \\
    &\ge r_h(s,a)+\rho_h(V^*_{h+1},P_{h}(s,a)) = Q^*_h(s,a).
\end{align*}
The first inequality is due to the monotonicity of $\rho_h$ and the induction hypothesis, and the second one follows from that
\begin{align*}
    \rho_h(V^*_{h+1},P_{h}(s,a))-\rho_h(V^*_{h+1},\hat{P}^k_{h}(s,a))  &\leq L_{\infty}(\rho_h,H-h)\norm{(V^*_{h+1},P_{h}(s,a))- \prt{ V^*_{h+1},\hat{P}^k_{h}(s,a) } }_{\infty}\\
    &\leq L_{\infty}(\rho_h,H-h)\sqrt{\frac{\iota}{2(N^k_h(s,a)\vee 1)}}=b^k_h(s,a),
\end{align*}
where the first inequality follows from the Lipschitz property of $\rho_h$ and Lemma \ref{lem:range}, and the second one is due to the DKW inequality (Fact \ref{fct:dkw}).
\end{proof}
\begin{fact}
\label{fct:inf_1}
Let $(x,P)$ and $(x,Q)$ be two discrete distributions with the same support $x=(x_1,\cdots,x_m)$ and $F,G$ be their CDFs respectively. It holds that
\[ \norm{F-G}_{\infty}\leq \norm{P-Q}_1. \]
\end{fact}
\begin{proof}
Without loss of generality, we assume that $x_1\leq x_2\cdots \leq x_n$. By definition, 
\begin{align*}
    \norm{F-G}_{\infty}&=\sup_{x\in\bR}\abs{F(x)-G(x)}=\max_{i\in[n]}\abs{F(x_i)-G(x_i)}=\max_{i\in[n]}\abs{\sum_{j\in[i]} P_j - \sum_{j\in[i]} Q_j } \\
    &\leq \max_{i\in[n]}\sum_{j\in[i]} \abs{P_j - Q_j}=\sum_{j\in[n]} \abs{P_j - Q_j} = \norm{P-Q}_1.
\end{align*}
The second equality is due to that $F$ and $G$ are piecewise constant functions that only differ at $x_1,\cdots,x_n$. This would lead a worse bonus term with a factor of $\sqrt{S}$.
\end{proof}
\begin{remark}
    Alternatively, we have
    \begin{align*}
    \rho_h(V^*_{h+1},P_{h}(s,a))-\rho_h(V^*_{h+1},\hat{P}^k_{h}(s,a))     &\leq L_{\infty}(\rho_h,H-h)\norm{(V^*_{h+1},P_{h}(s,a))- \prt{ V^*_{h+1},\hat{P}^k_{h}(s,a) } }_{\infty}\\
    &\leq L_{\infty}(\rho_h,H-h) \norm{P_{h}(s,a)-\hat{P}^k_{h}(s,a)}_1 \\
    &\leq L_{\infty}(\rho_h,H-h)\sqrt{\frac{2S}{(N^k_h(s,a)\vee 1)}\iota},
\end{align*}
where the second inequality is due to Fact \ref{fct:inf_1}, the third inequality is due to Lemma \ref{lem:main_event}.
\end{remark}

\paragraph{Step 2: regret decomposition.}
We introduce the key technical lemma here.
\begin{lemma}
\label{lem:1_1}
Let $(x,P)$ and $(y,P)$ be two discrete distributions, where $x=(x_1,\cdots,x_n)$ and $y=(y_1,\cdots,y_n)$. It holds that
\[ \norm{(x,P)-(y,P)}_{1}\leq \sum_{i\in[n]}P_i\abs{x_i-y_i}. \]
\end{lemma}
\begin{proof}
By the definition of  Wasserstein distance between two discrete distributions, we have 
\begin{align*}
    \norm{F-G}_1&=\inf_{\sum_{j}\lambda_{i,j}=P_i, \sum_{i}\lambda_{i,j}=P_j}\sum_i \sum_j \lambda_{i,j}\abs{x_i-y_j}\\
    &\leq \sum_i \sum_j \delta_{i,j}P_i\abs{x_i-y_j}\\
    &= \sum_i P_i\sum_j \delta_{i,j}\abs{x_i-y_i}\\
    &= \sum_i P_i\abs{x_i-y_i}.
\end{align*}
The inequality holds since $\{\delta_{i,j}P_i\}_{i,j}$ is a valid coupling 
\[  \sum_{j}\delta_{i,j}P_i=P_i,\ \sum_{i}\delta_{i,j}P_i=P_j.    \]
\end{proof}
We define $\Delta^k_h\triangleq V^k_h-V_h^{\pi^k} \in [-(H+1-h),H+1-h]^{\cS}$ 
and $\delta^k_h\triangleq\Delta^k_h(s^k_h)$. For any $(s,h)$ and any $\pi$, we let $P^{\pi}_h(\cdot|s)\triangleq P_h(\cdot|s,\pi_h(s))$. 
The regret can be bounded as
\begin{align*}
\text{Regret}(K)&=\sum^K_{k=1}V^{*}_1(s^k_1)-V^{\pi^k}_1(s^k_1)= \sum^K_{k=1}V^{*}_1(s^k_1)-V^{k}_1(s^k_1)+V^{k}_1(s^k_1)-V^{\pi^k}_1(s^k_1)\\
&\leq \sum^K_{k=1}V^{k}_1(s^k_1)-V^{\pi^k}_1(s^k_1)=\sum^K_{k=1}\delta^k_1.
\end{align*}
For simplicity, we write $r^k_h\triangleq r_h(s^k_h,\pi^k_h(s^k_h))$, $b^k_h\triangleq b^k_h(s^k_h,\pi^k_h(s^k_h))$, $N^k_h\triangleq N^k_h(s^k_h,\pi^k_h(s^k_h))$ and $\hat{P}^k_h(s^k_h)\triangleq \hat{P}^k_h(s^k_h,\pi^k_h(s^k_h))$. For any $h\in[H-1]$, we decompose $\delta^k_h$ as follows
\begin{align*}
    \delta^k_h&=\rho_h\prt{V^k_{h+1},\hat{P}^k_h(s^k_h)}+b^k_h-\rho_h\prt{V^{\pi^k}_{h+1},P^{\pi^k}_h(s^k_h)}\\
    &= \underbrace{ \rho_h\prt{V^k_{h+1},\hat{P}^k_h(s^k_h)}-\rho_h\prt{V^k_{h+1},P^{\pi^k}_h(s^k_h)} }_{(a)} + \underbrace{ \rho_h\prt{V^k_{h+1},P^{\pi^k}_h(s^k_h)}-\rho_h\prt{V^{\pi^k}_{h+1},P^{\pi^k}_h(s^k_h)} }_{(b)}+b^k_h.
\end{align*}
Using the Lipschitz property of $\rho_h$,
\begin{align*}
(a) &\leq  L_{\infty}(\rho_h,H-h)\norm{ \prt{V^k_{h+1},\hat{P}^k_h(s^k_h)} - \prt{V^k_{h+1},P^{\pi^k}_h(s^k_h)} }_{\infty}\\
&\leq L_{\infty}(\rho_h,H-h) \norm{\hat{P}^k_h(s^k_h)-P^{\pi^k}_h(s^k_h)}_1 \\
&\leq L_{\infty}(\rho_h,H-h)\sqrt{\frac{2S}{N^k_h\vee1}\iota}.
\end{align*}
Applying Lemma \ref{lem:1_1} yields that
\begin{align*}
    (b)&\leq L_{1}(\rho_h,H-h)\norm{ \prt{V^k_{h+1},P^{\pi^k}_h(s^k_h)} - \prt{V^{\pi^k}_{h+1},P^{\pi^k}_h(s^k_h)} }_{1}\\
    &\leq L_{1}(\rho_h,H-h) \sum_{s^{\prime}\in\cS} P^{\pi^k}_h(s^{\prime}|s^k_h)\abs{V^k_{h+1}(s^{\prime})-V^{\pi^k}_{h+1}(s^{\prime})}\\
    &= L_{1}(\rho_h,H-h)\sum_{s^{\prime}\in\cS} P^{\pi^k}_h(s^{\prime}|s^k_h)\prt{V^k_{h+1}(s^{\prime})-V^{\pi^k}_{h+1}(s^{\prime})}\\
    &= L_{1}(\rho_h,H-h) \brk{P^{\pi^k}_h \Delta^k_{h+1}}(s^k_h)\\
    &\triangleq L_{1}(\rho_h,H-h) (\epsilon^k_h + \delta^k_{h+1}),
\end{align*}
where $\epsilon^k_h\triangleq[P^{\pi^k}_h \Delta^k_{h+1}](s^k_h) -\Delta^k_{h+1}(s^k_{h+1})$ is a  martingale difference sequence with $\epsilon^k_h \in [-2(H-h),2(H-h)]$ a.s. for all $(k,h)\in[K]\times[H]$. The first equality is due to that $V^k_{h+1}(s^{\prime})\ge V^*_{h+1}(s^{\prime})\ge V^{\pi^k}_{h+1}(s^{\prime})$ for all $s^{\prime}$.

Now we can bound $\delta^k_h$ recursively
\begin{align*}
    &\delta^k_h \leq L_{\infty}(\rho_h,H-h)\sqrt{\frac{2S}{N^k_h\vee1}\iota}+ L_{1}(\rho_h,H-h) (\epsilon^k_h + \delta^k_{h+1}) 
    + L_{\infty}(\rho_h,H-h)\sqrt{\frac{\iota}{2(N^k_h\vee 1)}}\\
    &\leq 2L_{\infty}(\rho_h,H-h)e^k_h + L_{1}(\rho_h,H-h) (\epsilon^k_h + \delta^k_{h+1}),
\end{align*}
where we define $e^k_h\triangleq \sqrt{\frac{2S}{N^k_h\vee1}\iota}$ in the last line. Repeating the procedure, we obtain
\begin{align*}
    \delta^k_1 &\leq 2\sum^{H-1}_{h=1}L_{\infty}(\rho_h,H-h) \prod^{h-1}_{i=1} L_{1}(\rho_{i},H-i)e^k_{h}   + \sum^{H-1}_{h=1} \prod^{h}_{i=1} L_{1}(\rho_i,H-i) \epsilon^k_{h} + \prod^{H-1}_{h=1} L_{1}(\rho_h,H-h) \delta^k_{H}\\
    &=2\sum^{H-1}_{h=1}L_{\infty}(\rho_h,H-h) \prod^{h-1}_{i=1} L_{1}(\rho_{i},H-i)e^k_{h} + \sum^{H-1}_{h=1} \prod^{h}_{i=1} L_{1}(\rho_i,H-i) \epsilon^k_{h},
\end{align*}
where the last step is because $\delta^k_H=Q^k_H-Q^*_H=r_H-r_H=0$.
\paragraph{Step 3: putting together.}
The total regret is bounded as
\begin{align*}
    &\text{Regret}(K)\leq \sum_{k\in[K]} \delta^k_1 \leq 2\sum^{H-1}_{h=1}L_{\infty}(\rho_h,H-h) \prod^{h-1}_{i=1} L_{1}(\rho_{i},H-i)\sum_{k=1}^K e^k_{h} + \sum_{k=1}^K\sum^{H-1}_{h=1} \prod^{h}_{i=1} L_{1}(\rho_i,H-i) \epsilon^k_{h}.
\end{align*}
The first term can be bounded as
\begin{align*}
2\sum^{H-1}_{h=1}L_{\infty}(\rho_h,H-h) \prod^{h-1}_{i=1} L_{1}(\rho_{i},H-i)\sum_{k=1}^K e^k_{h}&=
2\sum^{H-1}_{h=1}L_{\infty}(\rho_h,H-h) \prod^{h-1}_{i=1} L_{1}(\rho_{i},H-i)\sum_{k=1}^K \sqrt{\frac{2S}{N^k_h\vee1}\iota}\\
&\leq 4\sum^{H-1}_{h=1}L_{\infty}(\rho_h,H-h) \prod^{h-1}_{i=1} L_{1}(\rho_{i},H-i)\sqrt{S^2AK\iota}\\
&\triangleq 4\sum^{H-1}_{h=1}L_{\infty,h} \prod^{h-1}_{i=1} L_{1,i}\sqrt{S^2AK\iota}\\ 
&\triangleq 4\sum^{H-1}_{h=1}L_{\infty,h}\tilde{L}_{1,h-1}\sqrt{S^2AK\iota}, 
\end{align*}
where we denote by $L_{\infty,h}=L_{\infty}(\rho_h,H-h)$ and $\tilde{L}_{1,h-1}=\prod^{h-1}_{i=1} L_{1,i}$ for simplicity. We can bound the second term by Azuma-Hoeffding inequality: with probability at least $1-\delta^{\prime}$, the following holds
\begin{align*}
    \sum_{k=1}^K\sum^{H-1}_{h=1} \prod^{h}_{i=1} L_{1}(\rho_i,H-i) \epsilon^k_{h} &= \sum_{k=1}^K\sum^{H-1}_{h=1}  \tilde{L}_{1,h} \epsilon^k_{h} \leq \sqrt{\sum_{k=1}^K\sum^{H-1}_{h=1}\frac{(2(H-h)\tilde{L}_{1,h})^2}{2}\log(1/\delta^{\prime})} \\
    &=\sqrt{\sum^{H-1}_{h=1}(H-h)^2\tilde{L}^2_{1,h}}\sqrt{2K\log(1/\delta^{\prime})}
\end{align*} 
Using a union bound and let $\delta=\delta^{\prime}=\frac{\tilde{\delta}}{2}$, we have that with probability at least $1-\delta$, 
\begin{align*}
\text{Regret}(K)&\leq 4\sum^{H-1}_{h=1}L_{\infty,h}\tilde{L}_{1,h-1}\sqrt{S^2AK\iota} + \sqrt{\sum^{H-1}_{h=1}(H-h)^2\tilde{L}^2_{1,h}}\sqrt{2K\log(1/\delta^{\prime})} \\
&=\tilde{\mathcal{O}}\prt{\sum^{H-1}_{h=1}L_{\infty,h}\tilde{L}_{1,h-1}\sqrt{S^2AK }}.
\end{align*}
The equality is due to that
\begin{align*}
    \sum^{H-1}_{h=1}L_{\infty,h}\tilde{L}_{1,h-1} &\ge \sqrt{ \sum^{H-1}_{h=1}L^2_{\infty,h}\tilde{L}^2_{1,h-1} }= \sqrt{ \sum^{H-1}_{h=1}((H-h)L_{1,h})^2\tilde{L}^2_{1,h-1} }\\
    &= \sqrt{ \sum^{H-1}_{h=1}(H-h)^2\tilde{L}^2_{1,h} },
\end{align*}
where the first inequality comes from the non-negativity of $L_{\infty,h}\tilde{L}_{1,h-1}$, and the first equality is due to the choice $L_{\infty,h}=L_{1,h}(H-h)$.
\begin{remark}
The following statement is not true
\[ \norm{(x,P)-(y,P)}_1 \leq \abs{ \sum_{i\in[n]}P_i (x_i-y_i) }.     \]
Consider the case that $(x,P)=((0,1),(\frac{1}{3},\frac{2}{3})$ and $(y,P)=((\frac{1}{3},\frac{5}{6}),(\frac{1}{3},\frac{2}{3}))$. A simple calculation yields that $\sum_{i\in[n]}P_i (x_i-y_i)=0$.
\end{remark}

\subsubsection{Proof for Algorithm \ref{alg:OVI-DRM}}
\paragraph{Step 1: verify optimism.}
\begin{lemma}[Optimistic value function]
\label{lem:main_opt_om} 
Conditioned on event $\mathcal{G}_{\delta}$, the sequence  $\{V^k_1(s^k_1)\}_{k\in[K]}$ produced by Algorithm  \ref{alg:OVI-DRM} satisfies $V^k_1(s^k_1) \ge V^*_1(s^k_1), \forall k \in [K]$.
\end{lemma}
\begin{proof}
The proof follows from Appendix \ref{app:sub}.
\end{proof}
\paragraph{Step 2: regret decomposition.}
The regret can be bounded as
\begin{align*}
\text{Regret}(K)&=\sum^K_{k=1}V^{*}_1(s^k_1)-V^{\pi^k}_1(s^k_1)\leq \sum^K_{k=1}V^{k}_1(s^k_1)-V^{\pi^k}_1(s^k_1)=\sum^K_{k=1}\delta^k_1.
\end{align*}
For any $h\in[H-1]$, we decompose $\delta^k_h$ as follows
\begin{align*}
    \delta^k_h&=\rho_h\prt{V^k_{h+1},\tilde{P}^k_h(s^k_h)}-\rho_h\prt{V^{\pi^k}_{h+1},P^{\pi^k}_h(s^k_h)}\\
    &= \underbrace{ \rho_h\prt{V^k_{h+1},\tilde{P}^k_h(s^k_h)}-\rho_h\prt{V^k_{h+1},P^{\pi^k}_h(s^k_h)} }_{(a)} + \underbrace{ \rho_h\prt{V^k_{h+1},P^{\pi^k}_h(s^k_h)}-\rho_h\prt{V^{\pi^k}_{h+1},P^{\pi^k}_h(s^k_h)} }_{(b)}.
\end{align*}
Using the Lipschitz property of $\rho_h$,
\begin{align*}
(a) &\leq  L_{\infty}(\rho_h,H-h)\norm{ \prt{V^k_{h+1},\tilde{P}^k_h(s^k_h)} - \prt{V^k_{h+1},P^{\pi^k}_h(s^k_h)} }_{\infty}\\
&\leq L_{\infty}(\rho_h,H-h) \norm{\tilde{P}^k_h(s^k_h)-P^{\pi^k}_h(s^k_h)}_1 \\
&\leq L_{\infty}(\rho_h,H-h)  \prt{ \norm{\tilde{P}^k_h(s^k_h)-\hat{P}^{\pi^k}_h(s^k_h)}_1 + \norm{\hat{P}^k_h(s^k_h)-P^{\pi^k}_h(s^k_h)}_1 } \\
&\leq 2L_{\infty}(\rho_h,H-h)\sqrt{\frac{2S}{N^k_h\vee1}\iota}.
\end{align*}
Using arguments similar to the proof for Algorithm  \ref{alg:UCBVI-DRM}
\begin{align*}
    (b)&\leq  L_{1}(\rho_h,H-h) (\epsilon^k_h + \delta^k_{h+1}),
\end{align*}
Now we can bound $\delta^k_h$ recursively
\begin{align*}
    \delta^k_h &\leq 2L_{\infty}(\rho_h,H-h)\sqrt{\frac{2S}{N^k_h\vee1}\iota}+ L_{1}(\rho_h,H-h) (\epsilon^k_h + \delta^k_{h+1}) 
    \\
    &= 2L_{\infty}(\rho_h,H-h)e^k_h + L_{1}(\rho_h,H-h) (\epsilon^k_h + \delta^k_{h+1}).
\end{align*}
 Repeating the procedure, we obtain
\begin{align*}
    \delta^k_1 &\leq 2\sum^{H-1}_{h=1}L_{\infty}(\rho_h,H-h) \prod^{h-1}_{i=1} L_{1}(\rho_{i},H-i)e^k_{h} + \sum^{H-1}_{h=1} \prod^{h}_{i=1} L_{1}(\rho_i,H-i) \epsilon^k_{h}.
\end{align*}
\paragraph{Step 3: putting together.}
The results follows from analogous arguments of the proof for Algorithm \ref{alg:UCBVI-DRM}.
\subsection{Gap-dependent Regret Upper Bound}
\paragraph{Step 1: regret decomposition.}
The regret of each episode can be rewritten as the expected sum of sub-optimality gaps for each action:
\begin{align*}
    (V^*_1-V^{\pi^k}_1)(s^k_1)&=V^*_1(s^k_1)-Q^*_1(s^k_1,a^k_1)+(Q^*_1-Q^{\pi^k}_1)(s^k_1,a^k_1)\\
    &=\Delta_1(s^k_1,a^k_1)+ [P_2(V^*_2-V^{\pi^k}_2)](s^k_2,a^k_2)\\
    &=\cdots=\bE\brk{\sum_{h=1}^H \Delta_h(s^k_h,a^k_h)}.
\end{align*}
\paragraph{Step 2: optimism.}
\begin{lemma}
With probability at least $1-\delta$, the following event holds
\[ 0\leq(Q^k_h-Q^*_h)(s,a)\leq  2b^k_h(s,a)+L_{1,h}[P_h(V^k_{h+1}-V^*_{h+1})](s,a).       \]
\end{lemma}
\begin{proof}
\begin{align*}
    &(Q^k_h-Q^*_h)(s,a)=r_h(s,a)+\rho_h\prt{V^k_{h+1},\hat{P}^k_h(s,a)}+b^k_h(s,a)-r_h(s,a)-\rho_h\prt{V^*_{h+1},P_h(s,a)}\\
    &= \underbrace{ \rho_h\prt{V^k_{h+1},\hat{P}^k_h(s,a)}-\rho_h\prt{V^k_{h+1},P_h(s,a)} }_{(a)}+ \underbrace{ \rho_h\prt{V^k_{h+1},P_h(s,a)}-\rho_h\prt{V^*_{h+1},P_h(s,a)} }_{(b)}+b^k_h(s,a)\\
    &\leq L_{\infty,h}\norm{(V^k_{h+1},\hat{P}^k_h(s,a))-(V^k_{h+1},P_h(s,a))}_{\infty}+L_{1,h}\norm{(V^k_{h+1},P_h(s,a))-(V^*_{h+1},P_h(s,a))}_{1}+b^k_h(s,a)\\
    &\leq L_{\infty,h}\norm{\hat{P}^k_h(s,a)-V^k_{h+1},P_h(s,a)}_{1} + L_{1,h} [P_h(V^k_{h+1}-V^*_{h+1})](s,a)+b^k_h(s,a)\\
    &\leq 2b^k_h(s,a) + L_{1,h} [P_h(V^k_{h+1}-V^*_{h+1})](s,a)
\end{align*}
\end{proof}
\paragraph{Step 3: bound number of steps in each interval}
\begin{lemma}
For any $n\in[N]$, 
\[      C^n\triangleq \abs{ \cbrk{(k,h):(Q^k_h-Q^*_h)(s^k_h,a^k_h)\in[2^{n-1}\Delta_{\text{min}},2^{n}\Delta_{\text{min}})} } \leq \mathcal{O}\prt{  \frac{ HS^2A\iota \prt{ \sum_{h^{\prime}=h}^{H-1} \prod_{i=h}^{h^{\prime}-1} L_{1,i}L_{\infty,h^{\prime}}}^2 }{4^n\Delta_{\text{min}}^2} }. \]
\end{lemma}
\begin{proof}
For every $n\in[N]$, $h\in[H]$, define 
\begin{align*}
    w^{(n,h)}_k &\triangleq \mathbb{I}\cbrk{(Q^k_h-Q^*_h)(s^k_h,a^k_h)\in [2^{n-1}\Delta_{\text{min}},2^{n}\Delta_{\text{min}}) } \\
    C^{(n,h)} &\triangleq \sum_{k=1}^K w^{(n,h)}_k.
\end{align*}
Observe that $w^{(n,h)}_k\leq1$ and $(w^{(n,h)}_k)^2=w^{(n,h)}_k$ . Now we bound $\sum_{k=1}^K w^{(n,h)}_k (Q^k_h-Q^*_h)(s^k_h,a^k_h)$ from both sides. On the one hand, by Lemma ,
\begin{align*}
    \sum_{k=1}^K w^{(n,h)}_k (Q^k_h-Q^*_h)(s^k_h,a^k_h)&\leq  4\sqrt{ S^2A\iota C^{(n,h)} }\cdot\sum_{h^{\prime}=h}^{H-1} \prod_{i=h}^{h^{\prime}-1} L_{1,i}L_{\infty,h^{\prime}}+
    \sqrt{2C^{(n,h)} \log\frac{1}{\delta^{\prime}}}\cdot \sum_{h^{\prime}=h}^{H-1} \prod_{i=h}^{h^{\prime}-1} L_{1,i}L_{1,h^{\prime}} (H-h^{\prime})\\
    &=\mathcal{O}\prt{\sqrt{ S^2A\iota C^{(n,h)} }\cdot\sum_{h^{\prime}=h}^{H-1} \prod_{i=h}^{h^{\prime}-1} L_{1,i}L_{\infty,h^{\prime}}}.
\end{align*}
On the other hand, by the definition of $w^{(n,h)}_k$,
\[  \sum_{k=1}^K w^{(n,h)}_k (Q^k_h-Q^*_h)(s^k_h,a^k_h) \ge \sum_{k=1}^K w^{(n,h)}_k 2^{n-1}\Delta_{\text{min}}=2^{n-1}\Delta_{\text{min}}\cdot C^{(n,h)}. \]
Combining the two inequalities, we obtain
\[  C^{(n,h)}\leq  \mathcal{O}\prt{  \frac{ S^2A\iota \prt{ \sum_{h^{\prime}=h}^{H-1} \prod_{i=h}^{h^{\prime}-1} L_{1,i}L_{\infty,h^{\prime}}}^2 }{4^n\Delta_{\text{min}}^2} }   \]
Finally, we have 
\[ C^{(n)}=\sum_{h=1}^H C^{(n,h)}\leq  \mathcal{O}\prt{  \frac{ S^2A\iota \sum_{h=1}^H\prt{ \sum_{h^{\prime}=h}^{H-1} \prod_{i=h}^{h^{\prime}-1} L_{1,i}L_{\infty,h^{\prime}}}^2 }{4^n\Delta_{\text{min}}^2} }\leq  \mathcal{O}\prt{  \frac{ S^2A\iota H\prt{ \sum_{h^{\prime}=1}^{H-1}  \tilde{L}_{1,h^{\prime}-1}L_{\infty,h^{\prime}}}^2 }{4^n\Delta_{\text{min}}^2} } \]
\end{proof}
\begin{lemma}
For any positive sequence $\{w_k\}_{k\in[K]}$, it holds that for any $h\in[H]$
\[ \sum_{k=1}^K w_k (Q^k_h-Q^*_h)(s^k_h,a^k_h)\leq  4\sqrt{w S^2A\iota \sum^K_{k=1}w_k }\cdot\sum_{h^{\prime}=h}^{H-1} \prod_{i=h}^{h^{\prime}-1} L_{1,i}L_{\infty,h^{\prime}}+
    \sqrt{2\sum_{k=1}^K w^2_k \log\frac{1}{\delta^{\prime}}}\cdot \sum_{h^{\prime}=h}^{H-1} \prod_{i=h}^{h^{\prime}-1} L_{1,i}L_{1,h^{\prime}} (H-h^{\prime}).\]
\end{lemma}
\begin{proof}
By Lemma 5,
\begin{align*}
    \sum_{k=1}^K w_k (Q^k_h-Q^*_h)(s^k_h,a^k_h) &\leq \sum_{k=1}^K w_k \prt{ 2 L_{\infty,h} \sqrt{\frac{2S\iota}{N^k_h}} + L_{1,h}[P_h(V^k_{h+1}-V^*_{h+1})](s^k_h,a^k_h) } \\
    &=  \underbrace{ 2 L_{\infty,h}\sum_{k=1}^K w_k \sqrt{\frac{2S\iota}{N^k_h\vee1}} }_{(a)}+  \underbrace{ L_{1,h}\sum_{k=1}^K w_k \epsilon^k_h }_{(b)}+L_{1,h}\sum_{k=1}^K w_k (V^k_{h+1}-V^*_{h+1})(s^k_{h+1})\\
    &\leq (a)+(b)+L_{1,h}\sum_{k=1}^K w_k (Q^k_{h+1}-Q^*_{h+1})(s^k_{h+1},a^k_{h+1}),
\end{align*}
where $\epsilon^k_h\triangleq [P_h(V^k_{h+1}-V^*_{h+1})](s^k_h,a^k_h)-(V^k_{h+1}-V^*_{h+1})](s^k_{h+1})\in[-2(H-h),2(H-h)]$ is a martingale difference sequence w.r.t. $\mathcal{F}^k_h$ for any $h\in[H]$, i.e., $\bE\brk{\epsilon^k_h|\mathcal{F}^k_h}=0$. Define $k(s,a,t)\triangleq\min\{k:N^k_h(s,a)\ge t\}$ the episode when $(s,a)$ is visited $t$ times at step $h$. We can bound term $(a)$ as
\begin{align*}
    (a)= 2 L_{\infty,h}\sum_{k=1}^K w_k \sqrt{\frac{2S\iota}{N^k_h}}&=2 L_{\infty,h}\sqrt{2S\iota}\sum_{s,a}\sum_{k=1}^K\mathbb{I}\{(s^k_h,a^k_h)=(s,a)\}\frac{w_k}{\sqrt{N^k_h(s,a)\vee1}}\\
    &=2 L_{\infty,h}\sqrt{2S\iota}\sum_{s,a}\sum_{t=1}^{N^K_h(s,a)}\frac{w_{k(s,a,t)}}{\sqrt{t}}\\
    &\leq 2 L_{\infty,h}\sqrt{2S\iota}\sum_{s,a}\sum_{t=1}^{ C_{s,a}/w }\frac{w}{\sqrt{t}}\\
    &\leq 4 L_{\infty,h}\sqrt{S\iota}\sum_{s,a}\sqrt{C_{s,a}w}\\
    &\leq 4 L_{\infty,h}\sqrt{w S^2A\iota \sum^K_{k=1}w_k },
\end{align*}
where $C_{s,a}\triangleq\sum_{t=1}^{N^K_h(s,a)} w_{k(s,a,t)}$ and $w_k\leq w$ for any $k$, and the last inequality follows from that $\sum_{s,a} C_{s,a}=\sum_{s,a}\sum_{t=1}^{N^K_h(s,a)} w_{k(s,a,t)}=\sum^K_{k=1}w_k$.

Since $\{\epsilon^k_h\}_{k\in[K]}$ is a MDS with $|\epsilon^k_h|\leq 2(H-h)$, we can bound term $(b)$ by Azuma-Hoeffding inequality: w.p. $1-\delta^{\prime}$
\begin{align*}
    (b)=L_{1,h}\sum_{k=1}^K w_k \epsilon^k_h\leq L_{1,h} (H-h)\sqrt{2\sum_{k=1}^K w^2_k \log\frac{1}{\delta^{\prime}}}.
\end{align*}
Thus we can get a recursive bound
\begin{align*}
    \sum_{k=1}^K w_k (Q^k_h-Q^*_h)(s^k_h,a^k_h)\leq  4 L_{\infty,h}\sqrt{w S^2A\iota \sum^K_{k=1}w_k }+ L_{1,h} (H-h)\sqrt{2\sum_{k=1}^K w^2_k \log\frac{1}{\delta^{\prime}}} \\
    + L_{1,h}\sum_{k=1}^K w_k (Q^k_{h+1}-Q^*_{h+1})(s^k_{h+1},a^k_{h+1}).
\end{align*}

Unrolling the inequality yields 
\begin{align*}
    &\sum_{k=1}^K w_k (Q^k_h-Q^*_h)(s^k_h,a^k_h)\leq  \sum_{h^{\prime}=h}^{H-1} \prod_{i=h}^{h^{\prime}-1} L_{1,i} \prt{ 4 L_{\infty,h^{\prime}}\sqrt{w S^2A\iota \sum^K_{k=1}w_k }+ L_{1,h^{\prime}} (H-h^{\prime})\sqrt{2\sum_{k=1}^K w^2_k \log\frac{1}{\delta^{\prime}}} }\\
    &= 4\sqrt{w S^2A\iota \sum^K_{k=1}w_k }\cdot\sum_{h^{\prime}=h}^{H-1} \prod_{i=h}^{h^{\prime}-1} L_{1,i}L_{\infty,h^{\prime}}+
    \sqrt{2\sum_{k=1}^K w^2_k \log\frac{1}{\delta^{\prime}}}\cdot \sum_{h^{\prime}=h}^{H-1} \prod_{i=h}^{h^{\prime}-1} L_{1,i}L_{1,h^{\prime}} (H-h^{\prime})
\end{align*}
\end{proof} 
\paragraph{Step 4: Bound the regret}
Denote by $\tau\triangleq(s^k_h,a^k_h)_{k,h}$ the trajectory. Define $\text{clip}[x|\delta]\triangleq x\mathbb{I}\{x\ge\delta\}$. Observe that
\[ V^*_h(s^k_h)=\max_{a} Q^*_h(s^k_h,a)\leq \max_{a} Q^k_h(s^k_h,a)=Q^k_h(s^k_h,a^k_h),\]
thus we get
\[  \Delta_h(s^k_h,a^k_h)=\text{clip}[V^*_h(s^k_h)-Q^*_h(s^k_h,a^k_h)|\Delta_{\text{min}}]\leq \text{clip}[(Q^k_h-Q^*_h)(s^k_h,a^k_h)|\Delta_{\text{min}}].\]
\begin{align*}
    \text{Regret}(K)&=\bE\brk{\sum_{k=1}^K\sum_{h=1}^H \Delta_h(s^k_h,a^k_h)}=\sum \bP(\tau)\sum_{k=1}^K\sum_{h=1}^H\Delta_h(s^k_h,a^k_h|\tau)\\
    &\leq \sum_{\tau\in E}\bP(\tau)\sum_{k=1}^K\sum_{h=1}^H\text{clip}[(Q^k_h-Q^*_h)(s^k_h,a^k_h|\tau)|\Delta_{\text{min}}]+
    \sum_{\tau\in E^c}\bP(\tau)KH^2\\
    &\leq \bP(E)\sum_{n=1}^N 2^n\Delta_{\text{min}} C^{(n)} + \bP(E^c)KH^2\\
    &\leq \sum_{n=1}^N \mathcal{O}\prt{  \frac{ HS^2A\iota \prt{ \sum_{h^{\prime}=h}^{H-1} \prod_{i=h}^{h^{\prime}-1} L_{1,i}L_{\infty,h^{\prime}}}^2 }{2^n\Delta_{\text{min}}} } +H\\
    &=\mathcal{O}\prt{  \frac{ HS^2A\iota \prt{ \sum_{h^{\prime}=h}^{H-1} \prod_{i=h}^{h^{\prime}-1} L_{1,i}L_{\infty,h^{\prime}}}^2 }{\Delta_{\text{min}}} }.
\end{align*}

\section{Proofs of Regret Lower Bounds}
\label{app:lb}
\subsection{Minimax Lower Bound}
We make the following assumption \cite{domingues2021episodic}.
\begin{assumption}
\label{asp:mini_lb}
Assume $S\ge6, A\ge 2$, and there exists an integer $d$ such that $S=3+\frac{A^d-1}{A-1}$. We further assume that $H\ge3d$ and $\bar{H}\triangleq\frac{H}{3}\ge1$.
\end{assumption}
\begin{theorem}
Assume Assumption \ref{asp:mini_lb} holds.  For any algorithm $\sA$, there exists an MDP $\mathcal{M}_{\sA}$ such that for sufficiently large $K$ 
\[ \mathbb{E}[\text{Regret}(\sA,\mathcal{M}_{\sA},K)]\ge \frac{\sqrt{p}}{27\sqrt{6}}c_{\rho,1}H\sqrt{SAT}.\]
\end{theorem}
\paragraph{Step 1.}
Fix an arbitrary algorithm $\sA$. We  introduce three types of special states for the hard MDP class: a waiting state $s_w$ where the agent starts and may stay until stage $\bar{H}$, after that it has to leave; a good state $s_g$, which is absorbing and is the only rewarding state; a bad state $s_b$ that is absorbing and provides no reward. The rest of $S-3$ states are part of a $A$-ary tree of depth $d-1$. The agent can only arrive $s_w$ from the root node $s_{root}$ and can only reach $s_g$ and $s_b$ from the leaves of the tree.
Let $\bar{H}\in[H-d]$ be the first parameter of the MDP class. We define $\tilde{H}:=\bar{H}+d+1$ and $H^{\prime}:=H+1-\tilde{H}$. We denote by  $\mathcal{L}:=\{s_1,s_2,...,s_{\bar{L}}\}$  the set of $\bar{L}$ leaves of the tree. For each $u^*:=\left(h^{*}, \ell^{*}, a^{*}\right) \in [d+1:\bar{H}+d]\times\mathcal{L}\times\mathcal{A}$, we define an MDP $\mathcal{M}_{u^*}$ as follows. The transitions in the tree are deterministic, hence taking action $a$ in state $s$ results in the $a$-th child of node $s$. The transitions from $s_w$ are defined as
	\[
	P_{h}\left(s_{\mathrm{w}} \mid s_{\mathrm{w}}, a\right) := \mathbb{I}\left\{a=a_{\mathrm{w}}, h \leq \bar{H}\right\} \quad \text { and } \quad P_{h}\left(s_{\text {root }} \mid s_{\mathrm{w}}, a\right) := 1-P_{h}\left(s_{\mathrm{w}} \mid s_{\mathrm{w}}, a\right).
	\]
	The transitions from any leaf $s_i\in\mathcal{L}$  are specified  as
	\[ P_{h}\left(s_{g} \mid s_{i}, a\right) := p+\Delta_{u^*}\left(h, s_{i}, a\right) \quad \text { and } \quad P_{h}\left(s_{b} \mid s_{i}, a\right) := 1- p-\Delta_{u^*}\left(h, s_{i}, a\right), \]
	where $\Delta_{u^*}\left(h, s_{i}, a\right):=\epsilon\mathbb{I}\{(h, s_{i}, a)=(h^{*}, s_{\ell^{*}}, a^{*})\}$ for some constants  $p\in[0,1]$ and $\epsilon\in[0,\min(1-p,p)]$ to be determined later. $p$ and $\epsilon$ are the second and third parameters of the MDP class. Observe that  $s_g$ and $s_b$ are absorbing, therefore we have $\forall a, P_{h}\left(s_{g} \mid s_{g}, a\right):=P_{h}\left(s_{b} \mid s_{b}, a\right):=1$. The reward  is a deterministic function of the state 
	\[   r_h(s,a):= \mathbb{I}\{s=s_g, h\ge\tilde{H}\}.\]
	Finally, we define a reference MDP $\mathcal{M}_0$ which differs from the previous MDP instances only in that $\Delta_0(h,s_i, a):=0$ for all $(h,s_i, a)$. For each $\epsilon,p$ and $\bar{H}$, we define the MDP class 
	\[  \mathcal{C}_{\bar{H},p,\epsilon}:=\mathcal{M}_0\cup\{\mathcal{M}_{u^*}\}_{u^*\in[d+1:\bar{H}+d]\times\mathcal{L}\times\mathcal{A}}.\]
	For an  MDP $\mathcal{M}_{u^*}$, the optimal policy $\pi^{*,\mathcal{M}_{u^*}}$  starts to traverse the tree at step $h^*-d$ then chooses to reach the leaf $s_{l^*}$ and performs action $a^*$. The optimal value in any of these MDPs is the same
	\begin{align*}
	    V^{*,\mathcal{M}_{u^*}}_1 &=V^{*,\mathcal{M}_{u^*}}_{h^*}(s_{l^*})=Q^{*,\mathcal{M}_{u^*}}_{h^*}(s_{l^*},a^*)=r_h(s_{l^*},a^*)+\rho_{h^*}(V^{*,\mathcal{M}_{u^*}}_{h^*+1},P_h(s_{l^*},a^*))\\
	    &=\rho_{h^*}((V^{*,\mathcal{M}_{u^*}}_{h^*+1}(s_g),V^{*,\mathcal{M}_{u^*}}_{h^*+1}(s_b)),(p+\epsilon,1-p-\epsilon)).
	\end{align*}
	For simplicity, we may drop $\mathcal{M}_{u^*}$ from the notations.	Notice that the agent must be in either of the absorbing states at step $h\ge \tilde{H}=\bar{H}+d+1$. Observe that $V^{*,\mathcal{M}_{u^*}}_{H}(s_g)=r_H(s_g,a)=1$ since $r_h(s_g,a)=1$ for any $a$ and any $h\ge \tilde{H}$, and $V^{*,\mathcal{M}_{u^*}}_{H}(s_b)=0$. Thus we have:
    \[ Q^{*,\mathcal{M}_{u^*}}_{H-1}(s_g,a)=r_{H-1}(s_g,a)+\rho_{H-1}((V^{*,\mathcal{M}_{u^*}}_{h^*+1}(s_g),V^{*,\mathcal{M}_{u^*}}_{h^*+1}(s_b)),(1,0))=1+V^{*,\mathcal{M}_{u^*}}_{h^*+1}(s_g)=2, \forall a, \]
    where the second equality follows from that $\rho_h(c)=c$ for a deterministic constant $c$.
    Therefore $V^{*,\mathcal{M}_{u^*}}_{H-1}(s_g)=2$. Similarly we can get $V^{*,\mathcal{M}_{u^*}}_{H-1}(s_b)=0$. It follows from inductions that $V^{*,\mathcal{M}_{u^*}}_{h}(s_g)=H+1-h$ and $V^{*,\mathcal{M}_{u^*}}_{h}(s_b)=0$ for $h\ge\tilde{H}$. Moreover, observe that
    \[ V^{*,\mathcal{M}_{u^*}}_{\tilde{H}-1}(s_g)=0+\rho_{\tilde{H}-1}((V^{*,\mathcal{M}_{u^*}}_{\tilde{H}}(s_g),V^{*,\mathcal{M}_{u^*}}_{\tilde{H}}(s_b)),(1,0))=V^{*,\mathcal{M}_{u^*}}_{\tilde{H}}(s_g)=H+1-\tilde{H}=H^{\prime}   \]
    and $V^{*,\mathcal{M}_{u^*}}_{\tilde{H}-1}(s_b)=0$. Then $V^{*,\mathcal{M}_{u^*}}_{h^*+1}(s_g)=\cdots=V^{*,\mathcal{M}_{u^*}}_{\tilde{H}-1}(s_g)=H^{\prime}$ and $V^{*,\mathcal{M}_{u^*}}_{h^*+1}(s_b)=\cdots=V^{*,\mathcal{M}_{u^*}}_{\tilde{H}-1}(s_b)=0$.
    Thus the optimal value 
    	\begin{align*}
	    V^{*,\mathcal{M}_{u^*}}_1   &=\rho_{h^*}((V^{*,\mathcal{M}_{u^*}}_{h^*+1}(s_g),V^{*,\mathcal{M}_{u^*}}_{h^*+1}(s_b)),(p+\epsilon,1-p-\epsilon))\\&=\rho_{h^*}((V^{*,\mathcal{M}_{u^*}}_{\tilde{H}}(s_g),V^{*,\mathcal{M}_{u^*}}_{\tilde{H}}(s_b)),(p+\epsilon,1-p-\epsilon))\\
	    &=\rho_{h^*}((H^{\prime},0),(p+\epsilon,1-p-\epsilon))
	\end{align*}
	Consider the case that policy $\pi^k\neq\pi^*$. Then we have $(s^k_{h^*},a^k_{h^*})\neq(s_{l^*},a^*)$.
	Analogously, we can get 
	\[  V^{\pi^k}_{h}(s_g)=H+1-h,\ V^{*,\mathcal{M}_{u^*}}_{h}(s_b)=0   \]
	for $h\ge\tilde{H}$. Suppose $\pi^k$ arrives at the leaf node $s^k_{l^k}$ in step $l^k\in [1+d,\tilde{H}-1]$, then $V^{\pi^k}_{l^k+1}(s_g)=\cdots=V^{\pi^k}_{\tilde{H}}(s_g)=H+1-\tilde{H}=H^{\prime}$ and  $V^{\pi^k}_{l^k+1}(s_b)=\cdots=V^{\pi^k}_{\tilde{H}}(s_b)=0$. Since $P_{l^k}(s_g|s^k_{l^k},a^k_h)=p$, 
    \begin{align*}
        V^{\pi^k}_1=\rho_{l^k}( (V^{\pi^k}_{l^k+1}(s_g),V^{\pi^k}_{l^k+1}(s_b)),(p,1-p) )=\rho_{l^k}( (H^{\prime},0),(p,1-p) )
    \end{align*}
    Denote by $x^k_h:=(s^k_h,a^k_h)$ for each $(k,h)$, $x^*:=(s_{\ell^{*}},a^*)$ and $N_K(u^*):=\sum_{k=1}^K\mathbb{I}\{x^k_{h^*}=x^*\}$. It follows that
    \[ V^{\pi^k}_1=\mathbb{I}\{x^k_{h^*}=x^*\}V^*_1+\mathbb{I}\{x^k_{h^*}\neq x^*\}\rho_{l^k}( (H^{\prime},0),(p,1-p) ) \]
    Define  $c_{\rho}$ as the constant that satisfies 
	\begin{align*}
	    \rho((H,0),(p,1-p))-\rho((H,0),(q,1-q))&\ge c_{\rho} \norm{((H,0),(p,1-p))-((H,0),(q,1-q))}_1 \\
     &= c_{\rho} H|p-q|.
	\end{align*}
    The expected regret of $\sA$ in $\mathcal{M}_{u^*}$ can be bounded  as 
	\begin{align*}
	&\ \ \ \ \bE_{\sA,\cM_{u^*}}[\text{Regret}(\sA,\mathcal{M}_{u^*},K)]\\
	&=\bE_{\sA,\cM_{u^*}}\brk{ \sum^K_{k=1} V^{*}_1-V^{\pi^k}_1 } \\
	&=\bE_{\sA,\cM_{u^*}}\brk{ \sum_{k=1}^K \mathbb{I}\{x^k_{h^*}\neq x^*\} \prt{ \rho_{h^*}((H^{\prime},0),(p+\epsilon,1-p-\epsilon)) - \rho_{l^k}( (H^{\prime},0),(p,1-p) ) } } \\
	&\ge \bE_{\sA,\cM_{u^*}}\brk{ \sum_{k=1}^K \mathbb{I}\{x^k_{h^*}\neq x^*\} c_{\rho}\norm{((H^{\prime},0),(p+\epsilon,1-p-\epsilon))-((H^{\prime},0),(p,1-p))}_{1} } \\
	&=\bE_{\sA,\cM_{u^*}}\brk{ \sum_{k=1}^K \mathbb{I}\{x^k_{h^*}\neq x^*\} c_{\rho}H^{\prime}\epsilon }\\
	&=c_{\rho}\epsilon H^{\prime}(K-\mathbb{E}_{\sA,\cM_{u^*}}[N_K(u^*)]),
	\end{align*}
    \paragraph{Step 2.} The maximum of the regret can be bounded below by the mean over all instances as
	\begin{align*}
	\max_{u^*\in[d+1:\bar{H}+d]\times\mathcal{L}\times\mathcal{A}}\text{Regret}(\sA,\mathcal{M}_{u^*},K)&\ge\frac{1}{\bar{H}\bar{L}A}\sum_{u^*\in[d+1:\bar{H}+d]\times\mathcal{L}\times\mathcal{A}}\text{Regret}(\sA,\mathcal{M}_{u^*},K)\\
	&\ge c_{\rho,1}H^{\prime}K\epsilon\prt{1-\frac{1}{\bar{L}AK\bar{H}}\sum_{u^*\in[d+1:\bar{H}+d]\times\mathcal{L}\times\mathcal{A}}\mathbb{E}_{u^*}[N_K(u^*)]}.
	\end{align*}
	Observe that it can be further bounded if we can obtain an upper bound on $\sum_{u^*\in[d+1:\bar{H}+d]\times\mathcal{L}\times\mathcal{A}}\mathbb{E}_{u^*}[N_K(u^*)]$, which can be done by relating each expectation to the expectation  under the reference MDP $\mathcal{M}_0$. \\
    \begin{fact}[Lemma 1,\cite{garivier2019explore}]
	\label{fct:fund_inq}
	Consider a measurable space $(\Omega, \mathcal{F})$ equipped with two distributions $\mathbb{P}_{1}$ and $\mathbb{P}_{2}$. For any $\mathcal{F}$-measurable function $Z: \Omega \rightarrow[0,1]$, we have
	$$
	\mathrm{KL}\left(\mathbb{P}_{1}, \mathbb{P}_{2}\right) \geq \mathrm{kl}\left(\mathbb{E}_{1}[Z], \mathbb{E}_{2}[Z]\right),
	$$
	where $\mathbb{E}_{1}$ and $\mathbb{E}_{2}$ are the expectations under $\mathbb{P}_{1}$ and $\mathbb{P}_{2}$ respectively.
    \end{fact}

    \begin{fact}[Lemma 5, \cite{domingues2021episodic}]
	\label{fct:div_dec}
	Let $\mathcal{M}$ and $\mathcal{M}^{\prime}$ be two MDPs that are identical except for their transition probabilities, denoted by $ P_{h}$ and $P_{h}^{\prime}$, respectively. Assume that we have $\forall(s, a)$, $P_{h}(\cdot \mid s, a) \ll P_{h}^{\prime}(\cdot \mid s, a).$ Then, for any stopping time $\tau$ with respect to $\left(I_k\right)_{k \geq 1}$ that satisfies $\mathbb{P}_{\mathcal{M}}[\tau<\infty]=1$
	$$
	\mathrm{KL}\left(\mathbb{P}_{\mathcal{M}}, \mathbb{P}_{\mathcal{M}^{\prime}}\right)=\sum_{(s,a,h) \in\mathcal{S}\times\mathcal{A}\times[H-1]} \mathbb{E}_{\mathcal{M}}\left[N^{\tau}_h(s,a)\right] \mathrm{KL}\left(P_{h}(\cdot \mid s, a), P_{h}^{\prime}(\cdot \mid s, a)\right).
	$$
    \end{fact}
    \begin{fact}[Lemma 28, \cite{liang2022bridging}]
	\label{fct:kl_bd}
	If $\epsilon\ge0$, $p\ge0$ and $p+\epsilon\in[0,\frac{1}{2}]$, then $\operatorname{kl}(p,p+\epsilon)\leq \frac{\epsilon^2}{2p(1-p)}\leq\frac{\epsilon^2}{p}$.
    \end{fact}
    By applying Fact \ref{fct:fund_inq} with $Z=\frac{N_K(u^*)}{K}\in[0,1]$, we have
	\[ \operatorname{kl}\prt{\frac{1}{K} \mathbb{E}_{0}\left[N_K(u^*)\right], \frac{1}{K} \mathbb{E}_{u^*}[N_K(u^*)]} \leq \operatorname{KL}\left(\mathbb{P}_{0}, \mathbb{P}_{ u^*}\right). \]
	By Pinsker's inequality, it implies that
	\[ \frac{1}{K} \mathbb{E}_{u^*}[N_K(u^*)]\leq \frac{1}{K} \mathbb{E}_{0}\left[N_K(u^*)\right] + \sqrt{\frac{1}{2}\operatorname{KL}\left(\mathbb{P}_{0}, \mathbb{P}_{ u^* }\right)}. \]
	Since $\mathcal{M}_0$ and $\mathcal{M}_{u^*}$ only differs at stage $h^*$ when $(s,a)=x^*$, it follows from Fact \ref{fct:div_dec} that
	\[ \operatorname{KL}\left(\mathbb{P}_{0}, \mathbb{P}_{u^*}\right)=\mathbb{E}_{0}\left[N_{K}(u^*)\right] \mathrm{kl}(p,p+\varepsilon). \]
	By Fact \ref{fct:kl_bd}, we have $\operatorname{kl}(p,p+\epsilon)\leq\frac{\epsilon^2}{p}$ for $\epsilon\ge0$ and $p+\epsilon\in[0,\frac{1}{2}]$. Consequently,
	\begin{align*}
	&\ \ \ \ \frac{1}{K} \sum_{u^*\in[d+1:\bar{H}+d]\times\mathcal{L}\times\mathcal{A}}\mathbb{E}_{u^*}[N_K(u^*)]\\
	&\leq \frac{1}{K} \mathbb{E}_{0}\left[\sum_{u^*\in[d+1:\bar{H}+d]\times\mathcal{L}\times\mathcal{A}}N_K(u^*)\right] + \frac{\epsilon}{\sqrt{2p}}\sum_{u^*\in[d+1:\bar{H}+d]\times\mathcal{L}\times\mathcal{A}}\sqrt{\mathbb{E}_{0}\left[N_{K}(u^*)\right]}\\
	&\leq 1+\frac{\epsilon}{\sqrt{2p}}\sqrt{\bar{L}AK\bar{H}},
	\end{align*}
	where the second inequality is due to the Cauchy-Schwartz inequality and that $\sum_{u^*\in[d+1:\bar{H}+d]\times\mathcal{L}\times\mathcal{A}}N_K(u^*)=K$.\\
	It follows that 
	\begin{align*}
	\max_{u^*\in[d+1:\bar{H}+d]\times\mathcal{L}\times\mathcal{A}}\text{Regret}(\sA,\mathcal{M}_{u^*},K)\ge c_{\rho,1}H^{\prime}K\epsilon\prt{1-\frac{1}{\bar{L}A\bar{H}}-\frac{\frac{\epsilon}{\sqrt{2p}}\sqrt{\bar{L}AK\bar{H}}}{\bar{L}A\bar{H}}}.
	\end{align*}
	\paragraph{Step 3.}
	Choosing $\epsilon=\sqrt{\frac{p}{2}}(1-\frac{1}{LA\bar{H}})\sqrt{\frac{LA\bar{H}}{K}}$ maximizes the lower bound
	\begin{align*}
	\max_{u^*\in[d+1:\bar{H}+d]\times\mathcal{L}\times\mathcal{A}}\text{Regret}(\sA,\mathcal{M}_{u^*},K)\ge \frac{\sqrt{p}}{2\sqrt{2}}c_{\rho,1}H^{\prime}\prt{1-\frac{1}{\bar{L}A\bar{H}}}^2\sqrt{\bar{L}AK\bar{H}}.
	\end{align*}
	Since $S\ge6$ and $A\ge2$, we have $\bar{L}=(1-\frac{1}{A})(S-3)+\frac{1}{A}\ge\frac{S}{4}$ and $1-\frac{1}{\bar{L}A\bar{H}}\ge1-\frac{1}{\frac{6}{4}\cdot2}=\frac{2}{3}$. Choose $\bar{H}=\frac{H}{3}$ and use the assumption that $d\leq\frac{H}{3}$ to obtain that $H^{\prime}=H-d-\bar{H}\ge\frac{H}{3}$. Now we choose arbitrary $p\leq\frac{1}{4}$
	and $\epsilon=\sqrt{\frac{p}{2}}(1-\frac{1}{\bar{L}A\bar{H}})\sqrt{\frac{LA\bar{H}}{K}}<\frac{1}{2\sqrt{2}}\sqrt{\frac{\bar{L}A\bar{H}}{K}}\leq\frac{1}{4}$ if $K\ge2\bar{L}A\bar{H}$. Such choice of $p$ and $\epsilon$ guarantees the assumption of Fact \ref{fct:kl_bd}. Finally we use the fact that $\sqrt{\bar{L}AK\bar{H}}\ge\frac{1}{2\sqrt{3}}\sqrt{SAKH}$ to obtain
	\[ \max_{u^*\in[d+1:\bar{H}+d]\times\mathcal{L}\times\mathcal{A}}\text{Regret}(\sA,\mathcal{M}_{u^*},K)\ge\frac{\sqrt{p}}{27\sqrt{6}}c_{\rho}H\sqrt{SAKH}.\]

\subsection{Gap-dependent Lower Bound}
\begin{theorem}[Gap-dependent regret lower bound]
Let $S\ge2$ and $A\ge2$, and let $\{\delta_{s,a}\}_{{s,a}\in\cS\times\cA}\subset (0,\frac{H}{8})$ denote a set of gaps. For any $h\ge1$, there exists an MDP $\cM$ with $\cS=[S+2]$ and $\cA=[A]$ such that any $\alpha$-uniformly good algorithm $\texttt{alg}$ satisfies
    \begin{align*}
        \lim_{K\rightarrow\infty} \frac{\text{Regret}(\texttt{alg},\cM,K)}{\log K}= \Omega\prt{(1-\alpha)\sum_{(s,a):\Delta_1(s,a)>0} \frac{ (c_{\rho} H)^2}{\Delta_1(s,a)} }
    \end{align*}
\end{theorem}
We first fix an arbitrary $\alpha$-uniformly good algorithm $\sA$. For simplicity, we may drop  $\sA$ from the notations, e.g., $\bE_{\cM}=\bE_{\cM,\sA}$. 
\paragraph{Step 1: construction of the hard instance.}
Our construction mirrors the lower bounds in . However, their instance is suited for homogeneous/stationary MDP. Define an MDP $\cM$ with $\cS=\{0\}\cup[S+2]$ and $\cA=[A]$. Without loss of generality, we consider the case $H\ge2$. Otherwise, it reduces to a bandit setting. We first specify the transition kernels. For the convenience of analysis, we introduce $s_0=0$ at stage $h=0$ with $P_0(s|0)=\frac{1}{S}$ for any $s\in[S]$. In other words, the initial state $s_1$ is uniformly distributed over $[S]$. For $(s,a)\in[S]\times[A]$, let
\[ P_1(S+1|s,a)=\frac{3}{4}-\frac{2\delta_{s,a}}{H-1}=:\frac{3}{4}-\tilde{\delta}_{s,a}, \ P_1(S+1|s,a)=1-P_1(S+1|s,a).  \]
Thus at stage 1, each state $s\in[S]$ can only transit to either state $S+1$ or $S+2$. Furthermore, we set state $S+1$ and $S+2$ to be absorbing state, i.e.
\[ P_h(S+1|S+1,a)=P_h(S+2|S+2,a)=1,\ \forall h\in[2:H-1], a\in[A]. \]
Finally, we set the reward functions as 
\[R(x, a):= \begin{cases} 1 & (x, a)=(S+1,1) \\ \frac{1}{2} & (x, a)=(S+2,1) \\ 0 & \text{otherwise}.\end{cases}\]
We assume that there exists a unique action $\pi^*(s)$ for each $s\in[S]$ such that $\delta_{s,\pi^*(s)}=0$. We will see that such action is the optimal action. Note that  $S+1$ and $S+2$ are absorbing states and the only two rewarding states, hence $V^*_h(S+1)=H+1-h$ and $V^*_h(S+2)=\frac{H+1-h}{2}$ for $h\in[2:H]$. It follows that 
for $x\in[S]$,
\begin{align*}
    V^*_h(s) &= 0 + \rho_h(V^*_{h+1}, P_h(s,\pi^*_h(s)))= \rho_h\prt{\prt{H-h,\frac{H-h}{2}}, \prt{\frac{3}{4},\frac{1}{4}}}, \\
    Q^*_h(s,a) &= 0 + \rho_h(V^*_{h+1}, P_h(s,a)) = \rho_h\prt{\prt{H-h,\frac{H-h}{2}}, \prt{\frac{3}{4}-\tilde{\delta}_{s,a},\frac{1}{4}+\tilde{\delta}_{s,a}}},
\end{align*}
which implies that
\begin{align*}
    \Delta_h(s,a) &=  \rho_h\prt{\prt{H-h,\frac{H-h}{2}}, \prt{\frac{3}{4},\frac{1}{4}}} - \rho_h\prt{\prt{H-h,\frac{H-h}{2}}, \prt{\frac{3}{4}-\tilde{\delta}_{s,a},\frac{1}{4}+\tilde{\delta}_{s,a}}} \\
                &\ge c_{\rho} \norm{ \prt{\prt{H-h,\frac{H-h}{2}}, \prt{\frac{3}{4},\frac{1}{4}}} - \prt{\prt{H-h,\frac{H-h}{2}}, \prt{\frac{3}{4}-\tilde{\delta}_{s,a},\frac{1}{4}+\tilde{\delta}_{s,a}} } }_1 \\
                &= c_{\rho} \frac{H-h}{2} \tilde{\delta}_{s,a}.
\end{align*}
In particular, we have $\Delta_1(s,a) \ge c_{\rho} \frac{H-1}{2} \tilde{\delta}_{s,a}=c_{\rho,1} \delta_{s,a}$. Note that $\Delta_1(s, a)$ is only defined for $s\in[S]$.
\paragraph{Step 2: regret decomposition.}
The regret for algorithm $\sA$ over MDP $\cM$ can be decomposed as follows 
\begin{align*}
    \text{Regret}(\sA,\cM,K) &= \bE \brk{\sum^K_{k=1} V^*_1(s^k_1)-V^{\pi^k}_1(s^k_1)} \\
    &= \bE \brk{\sum^K_{k=1} \sum_{h=1}^H \Delta_h(s^k_h,a^k_h)} \\
    & = \bE \brk{\sum^K_{k=1} \sum_{h=1}^H \sum_{s,a} \bI\cbrk{s^k_h=s,a^k_h=a} \Delta_h(s^k_h,a^k_h)}\\
    &=\sum_{h=1}^H\sum_{s,a}\bE \brk{\sum^K_{k=1}  \bI\cbrk{s^k_h=s,a^k_h=a}}\Delta_h(s,a) \\
    &= \sum_{h=1}^H\sum_{s,a}\bE \brk{N^K_h(s,a)}\Delta_h(s,a) \\
    &= \sum_{h=1}^H\sum_{s\in[S],a}\bE \brk{N^K_h(s,a)}\Delta_h(s,a).
\end{align*}
where the last equality is due to that $\Delta_h(s,a)$ is only defined over $[S]$. For our hard instance, observe that $\bI\cbrk{s^k_h=s}=0$ for $s\in[S]$ and $h\neq1$. Therefore $N^K_h(s,a)=0$ for $s\in[S]$ and $h\neq1$, which implies 
\[ \text{Regret}(\sA,\cM,K)= \sum_{s\in[S],a}\bE \brk{N^K_1(s,a)}\Delta_1(s,a).\]
We claim that for any $(s,a)$ such that $s\in[S]$ and $\Delta_1(s,a)>0$, and any $K\ge K_0(\cM)$, it holds that
\[ \bE_{\sA,\cM}\prt{N^K_1(s,a)} \ge \Omega \prt{\frac{1}{\tilde{\delta}_{s,a}^2}\log K} =  \Omega \prt{\frac{(c_{\rho,1} H)^2}{\Delta_1(s,a)^2}\log K}.\]
It follows that
\[ \text{Regret}(\sA,\cM,K)\ge \Omega \prt { \sum_{s\in[S],a:\Delta_1(s,a)>0} \frac{(c_{\rho,1} H)^2}{\Delta_1(s,a)}\log K}.\]
\paragraph{Step 3: bounding $\bE \brk{N^K_h(s,a)}$.}
Observe that $\tilde{\delta}_{s,a}=\frac{2}{H-1}\delta_{s,a}\in(0,\frac{1}{2})$ due to the assumption that $\delta_{s,a}\in(0,\frac{H}{8})$. By Fact \ref{fct:fund_inq} and Fact \ref{fct:div_dec}, let $Z$ be a $\cF^K$-measurable random variable, then it holds that
\[ \mathrm{kl}\left(\mathbb{E}_{\cM}[Z], \mathbb{E}_{\cM'}[Z]\right) \leq  \sum_{(s,a,h) \in\mathcal{S}\times\mathcal{A}\times[H-1]} \mathbb{E}_{\mathcal{M}}\left[N^{K}_h(s,a)\right] \mathrm{KL}\left(P_{h}(s, a), P_{h}^{\prime}(s, a)\right).\]
Now fix an arbitrary $(s,a)\in[S]\times[A]$. Define an MDP $\cM'_{s,a}$ which differs from $M$ only in that $P_1(S+1|s,a)=\frac{3}{4}+\eta$, where $\eta=\min\{\frac{1}{8},\tilde{\delta}_{s,a}\}$. For simplicity, we write $\cM'=\cM'_{s,a}$. The following holds
\[  \mathrm{kl}\left(\mathbb{E}_{\cM}[Z], \mathbb{E}_{\cM'}[Z]\right) \leq  \mathbb{E}_{\mathcal{M}}\left[N^{K}_1(s,a)\right] \mathrm{KL}\left(P_{1}(s, a), P_{1}^{\prime}(s, a)\right)=\mathbb{E}_{\mathcal{M}}\left[N^{K}_1(s,a)\right] \mathrm{kl}\left(\frac{3}{4}-\Tilde{\delta}_{s,a},\frac{3}{4}+\eta\right).   \]
Observe that $\frac{1}{4}<\frac{3}{4}-\Tilde{\delta}_{s,a}<\frac{3}{4}+\eta<\frac{7}{8}$ and $\eta+\Tilde{\delta}_{s,a}\leq2\Tilde{\delta}_{s,a}$, it follows from Fact \ref{fct:kl_bd} that
\[ \mathrm{kl}\left(\frac{3}{4}-\Tilde{\delta}_{s,a},\frac{3}{4}+\eta\right) \leq \frac{(\eta+\Tilde{\delta}_{s,a})^2}{2(\frac{3}{4}-\Tilde{\delta}_{s,a})(1-\frac{3}{4}-\eta)}<64 \Tilde{\delta}_{s,a}^2. \]
Now we have 
\[ \mathbb{E}_{\mathcal{M}}\left[N^{K}_1(s,a)\right] \ge \frac{1}{64 \Tilde{\delta}_{s,a}^2}\mathrm{kl}\left(\mathbb{E}_{\cM}[Z], \mathbb{E}_{\cM'}[Z]\right) \ge \frac{(c_{\rho,1}(H-1))^2}{256 \Delta_{1}^2(s,a)}\mathrm{kl}\left(\mathbb{E}_{\cM}[Z], \mathbb{E}_{\cM'}[Z]\right). \]
We set $Z = \sum_{k=1}^K \frac{\bI\{\pi^k_1(s)=a\}}{K}\in[0,1]$. Note that $Z$ is indeed $\cF^K$-measurable random variable since $\pi^k$ is $\cF^K$-measurable and $(s,a)$ is fixed. Denote by $\Delta'$ the gap for MDP $\cM'$. Observe that for $a'\neq a$, 
\begin{align*}
    \Delta_1'(s,a') &= \rho_h\prt{ \prt{H-1,\frac{H-1}{2}}, \prt{\frac{3}{4}+\eta,\frac{1}{4}-\eta} } - \rho_h\prt{ \prt{H-1,\frac{H-1}{2}}, \prt{\frac{3}{4} - \Tilde{\delta}_{s,a'},\frac{1}{4}+\Tilde{\delta}_{s,a'}} }\\
    &\ge c_{\rho} \frac{H-1}{2} (\eta+\Tilde{\delta}_{s,a'}) \ge c_{\rho} \frac{H-1}{2} \eta.
\end{align*}
Under MDP $\cM'$, action $a$ is the unique optimal action for $s$, thus
\begin{align*}
    \re  ((\sA,\cM',K) &\ge \sum_{a'\neq a} \bE_{\cM'}\brk{N^K_1(s,a')} \Delta'_1(s,a') \\
    &\ge c_{\rho} \frac{H-1}{2} \eta \sum_{a'\neq a} \bE_{\cM'}\brk{N^K_1(s,a')} \\
    &= c_{\rho} \frac{H-1}{2} \eta \bE_{\cM'}\brk{\sum_{k=1}^K\sum_{a'\neq a} \bI\prt{s^k_1=s,\pi^k_1(s^k_1)=a}}\\
    &=c_{\rho} \frac{H-1}{2} \eta \bE_{\cM'}\brk{\sum_{k=1}^K \prt{\bI\prt{s^k_1=s}-\bI\prt{s^k_1=s,\pi^k_1(s^k_1)=a}} } \\
    &= c_{\rho} \frac{H-1}{2} \eta \prt{ \frac{K}{S}- \bE_{\cM'}\brk{\sum_{k=1}^K\bI\prt{s^k_1=s} \bI\prt{\pi^k_1(s)=a} } } \\
    &= c_{\rho} \frac{H-1}{2} \eta \prt{ \frac{K}{S}- \sum_{k=1}^K\bE_{\cM'}\brk{\bI\prt{s^k_1=s}}\bE_{\cM'}\brk{ \bI\prt{\pi^k_1(s)=a} } } \\
    &= c_{\rho} \frac{H-1}{2} \frac{K}{S} \prt{ 1-  \bE_{\cM'}\brk{ Z } },
\end{align*}
where the second to the last equality is due to the dependence between $s^k_1$ and $\pi^k_1$. Since $\sA$ is $\alpha$-uniformly good algorithm, there exists $C_{\cM'}>0$ such that
\[ c_{\rho} \frac{H-1}{2} \frac{K}{S} \prt{ 1-  \bE_{\cM'}\brk{ Z } } \leq \re  ((\sA,\cM',K) \leq C_{\cM'} K^{\alpha}, \]
implying
\[ 1-  \bE_{\cM'}\brk{ Z } \leq \frac{2C_{\cM'}S}{c_{\rho,1}(H-1)K^{1-\alpha}}.\]
We can also get
\begin{align*}
    C_{\cM} K^{\alpha} \ge \re  ((\sA,\cM,K) &\ge \bE_{\cM}\brk{N^K_1(s,a)}\Delta_1(s,a) \ge \frac{K\Delta_1(s,a)}{S}\bE_{\cM}\brk{Z},
\end{align*}
which implies that $\bE_{\cM}\brk{Z} \leq \frac{C_{\cM}S}{\Delta_1(s,a)K^{1-\alpha}}$. By [], 
\[ \mathrm{kl}(x,y) \ge (1-x)\log \frac{1}{1-y}-\log 2. \]
It follows that 
\[ \mathrm{kl}\prt{\bE_{\cM}\brk{Z},\bE_{\cM'}\brk{Z} }\ge \prt{1-\frac{C_{\cM}S}{\Delta_1(s,a)K^{1-\alpha}}} \prt{(1-\alpha)\log K -\log \frac{2C_{\cM'}S}{c_{\rho,1}(H-1)}} -\log2.  \]
\paragraph{Step 4.} We can also prove for the case $h\neq 1$ by modifying the transition kernels for state 0. For $h\neq1$, we set the transition kernels as
\[  P_l(0|0,a)=1, \forall l\in[0:h-2], \forall a\in[A],\ P_{h-1}(s|0,a)=\frac{1}{S}, \forall s\in [S],\forall a\in[A].\]
In other words, the MDP is randomly initialized over $[S]$ at stage $h$ rather than stage 1. For $(s,a)\in[S]\times[A]$, let
\[ P_h(S+1|s,a)=\frac{3}{4}-\frac{2\delta_{s,a}}{H-1}=:\frac{3}{4}-\tilde{\delta}_{s,a}, \ P_h(S+1|s,a)=1-P_h(S+1|s,a).  \]
Finally, we still set $S+1, S+2$ to be absorbing states. Using similar arguments concludes the proof.

\end{document}